\pdfoutput=1
\documentclass{lmcs} 

\usepackage[utf8]{inputenc}
\usepackage[T1]{fontenc}
\usepackage[english]{babel}
\usepackage{hyperref} 
\usepackage{array}
\usepackage{longtable}
\usepackage{booktabs}
\usepackage{multirow}
\usepackage[dvipsnames]{xcolor}
\usepackage{tikz}
\usetikzlibrary{shapes, arrows, positioning, graphs, backgrounds}
\usepackage{graphicx} 
\usepackage{tkz-graph}
\usepackage{amsmath}
\usepackage{amssymb}
\usepackage{amsfonts}
\usepackage{stmaryrd} 
\usepackage{pifont} 
\usepackage[linesnumbered,ruled,vlined]{algorithm2e}
\usepackage{cleveref}
\usepackage{url}
\usepackage{multicol} 
\usepackage[normalem]{ulem}

\newcommand{\unit}{\ensuremath{\mathbf{U}}}

\newcommand{\kb}{\ensuremath{\mathcal{S}}}
\newcommand{\db}{\ensuremath{D}}
\newcommand{\pr}{\ensuremath{\mathbf{pr}}}

\newcommand{\C}{\ensuremath{\mathsf{C}}}
\newcommand{\PS}{\ensuremath{\mathsf{P}}}
\newcommand{\V}{\ensuremath{\mathsf{V}}}

\newcommand{\Dom}{\ensuremath{\mathbb{D}}}

\newcommand{\WCQ}{\ensuremath{\mathrm{N\mathbb{C}F}}}

\newcommand{\FreeConst}{\mathit{Fr}}

\newcommand{\mybox}{{\scriptsize\ensuremath{\hfill\blacksquare}\normalsize}}
\newcommand{\nop}[1]{}

\SetKwInput{KwInput}{Input}                
\SetKwInput{KwOutput}{Output}              

\keywords{Logic-based framework, Nexus of similarity, Knowledge bases, Computational complexity}

\begin{document}

\title[Navigating Taxonomic Expansions]{Navigating Taxonomic Expansions of Entity Sets Driven by Knowledge Bases}

\author[G.~Amendola]{Giovanni Amendola$^{(a)}$}
\author[P.~Cofone]{Pietro Cofone$^{(a)}$}
\author[M.~Manna]{Marco Manna$^{(a)}$}
\author[A.~Ricioppo]{Aldo Ricioppo$^{(a),(b)}$}

\address[a]{University of Calabria, Rende, Italy}
\email{name.surname@unical.it}

\address[b]{University of Cyprus, Nicosia, Cyprus}
\email{ricioppo.aldo@ucy.ac.cy}

\begin{abstract}
Recognizing similarities among entities is central to both human cognition and computational intelligence.
Within this broader landscape, Entity Set Expansion is one prominent task aimed at taking an initial set of (tuples of) entities and identifying additional ones that share relevant semantic properties with the former ---potentially repeating the process to form increasingly broader sets.
However, this ``linear'' approach does not unveil the richer ``taxonomic'' structures present in knowledge resources.
A recent logic-based framework introduces the notion of an \emph{expansion graph}: a rooted directed acyclic graph where each node represents a semantic generalization labeled by a logical formula, and edges encode strict semantic inclusion.
This structure supports {\em taxonomic expansions} of entity sets driven by knowledge bases.
Yet, the potentially large size of such graphs may make full materialization impractical in real-world scenarios.
To overcome this, we formalize reasoning tasks that check whether two tuples belong to comparable, incomparable, or the same nodes in the graph.
\textcolor{black}{Our results show that these tasks are intractable in general, that they remain intractable when the number of input tuples is bounded, and that they become solvable in polynomial time when the descriptions of the entities are small as well. The bounds we establish are tight.}
This enables local, incremental navigation of expansion graphs, supporting practical applications without requiring full graph construction.
\end{abstract}

\maketitle

\section{Introduction}\label{sec:intro}

\subsection{Context}

Similarities play a key role in many real-world scenarios, underpinning both cognitive processes and computational tasks~\cite{DBLP:journals/isci/HussainBWHJ23}. Humans routinely identify similarities and differences among entities or sets of entities to categorize objects, assess qualities, make informed choices, and stimulate creative thinking. At the heart of these activities lies the recognition of interconnected properties shared by entities, hereinafter \textcolor{black}{informally} referred to as \textcolor{black}{``commonalities'' or ``nexus of similarity''}.
This capacity extends to both tangible and intangible domains, shaping reasoning patterns and influencing behaviors in diverse contexts, from everyday judgments to strategic planning.

Researchers from various fields have proposed a variety of approaches to measure \emph{semantic similarity} between entities, typically expressed as descriptive ratings or numerical scores~\cite{GomaaFahmy13}. For instance, modern systems can assign a high similarity score to the pair $\langle\mathsf{Paris}\rangle$ and $\langle\mathsf{Rome}\rangle$ by taking into account somehow that both are ``European cities'', ``places situated on rivers'', ``capitals'', and so forth.

Inspired by ``Google Sets''~\cite{GoogleSets07}, considerable academic and commercial efforts have been also devoted to providing solutions for expanding a given set of entities with similar ones. The main tasks are:
\emph{entity set expansion}~\cite{DBLP:conf/emnlp/PantelCBPV09},
\emph{entity recommendation}~\cite{DBLP:conf/semweb/BlancoCMT13}, \emph{tuples expansion}~\cite{DBLP:conf/iiwas/ErAB16}, or \emph{entity suggestion}~\cite{DBLP:conf/ijcai/ZhangXHWWW17}. For example, one can expand the set $\unit = \{\langle\mathsf{Paris}\rangle, \langle\mathsf{Rome}\rangle\}$ and obtain $\unit' = \unit \cup \{\langle\mathsf{Amsterdam}\rangle\}$; then, one can reapply the process starting from $\unit'$ to obtain the set $\unit'' = \unit' \cup \{\langle\mathsf{Brussels}\rangle$,
$\langle\mathsf{Rio~de~Janeiro}\rangle$,
$\langle\mathsf{Vienna}\rangle\}$. Indeed, all these entities share one or more of the aforementioned properties, e.g., ``places situated on rivers'' or ``European cities''.

\textcolor{black}{Numerous complementary approaches across different fields have investigated the problem of recognizing and formally expressing, partially or entirely, the nexus of similarity among objects that have been defined through some underlying Knowledge Base (KB). Adopting a variety of settings, these approaches may differ in: the form of the input objects (e.g., pairs of entities, sets of entities, sets of entity tuples, sets of DL concepts, or sets of first-order expressions); the type of KB (e.g., TBox, relational database, RDF-graph, or first-order formulas); the scope of knowledge considered for each input item (e.g., the entire KB or selected excerpts); and the specific formalism used to express commonalities (e.g., DL concepts, rooted RDF-graphs, (U)CQs, rooted CQs, nearly-connected conjunctive formulas, pp-formulas, SPARQL queries, or first-order clauses).
Formulas expressing commonalities are hereinafter referred to as \emph{explanations}.
Among these, the most specific ones (i.e., those that explain ---in a comprehensive way--- all the nexus of similarity between the input objects) will be referred to as \emph{characterizations}.}

For example, to \textcolor{black}{explain} that the tuples of $\unit = \{\langle\mathsf{Paris}\rangle, \langle\mathsf{Rome}\rangle\}$ are ``European cities'',
one can use \textcolor{black}{expressions} of the following shapes:
\begin{gather*}
    \mathsf{city} \sqcap \exists \, \mathsf{located}.\{\mathsf{Europe}\} \\
    \text{or} \notag \\
    x \leftarrow \mathsf{city}(x), \mathsf{located}(x,\mathsf{Europe}) \\
    \text{or} \notag \\
    x \leftarrow \mathsf{isa}(x, \mathsf{city}), \mathsf{located}(x,\mathsf{Europe}).
\end{gather*}
\textcolor{black}{Expressions of this kind have been studied in several fields, each with its own formalism and its own name for characterizations.
In Inductive Logic Programming, objects are represented by clauses, that is, disjunctions of first-order literals, and their characterizations, which are clauses as well, are called \emph{least general generalizations}, a notion that goes back to the early 1970s~\cite{plotkin1971further,plotkin1972automatic}.} \textcolor{black}{Further studies incorporating global background knowledge followed over the years~\cite{DBLP:journals/ai/Buntine88,DBLP:journals/ngc/Muggleton91}.
In Description Logics (DLs), objects are represented by DL-concepts (resp., individuals occurring in the data), and characterizations are termed \emph{least common subsumers}~\cite{DBLP:conf/aaai/CohenBH92,DBLP:conf/ijcai/BaaderKM99} (resp., \emph{most specific concepts}~\cite{DBLP:conf/aaai/JungLW20}); here as well, an underlying KB (i.e., a TBox) can accompany the input~\cite{DBLP:journals/japll/BaaderST07}.
When the objects to be characterized are paired with a set of objects to be excluded, a generalized question is whether there exists a DL-concept that, in addition to characterizing the former, expresses nexus of similarity that do not hold for the objects in the second set~\cite{DBLP:conf/ilp/BadeaN00,DBLP:conf/ijcai/FunkJLPW19,DBLP:journals/ai/JungLPW22}.
Other settings in the Semantic Web treat objects as entities in a knowledge graph, where selected excerpts of knowledge are extracted for each input item, and characterizations range from rooted RDF-graphs to SPARQL queries featuring comparison operators~\cite{DBLP:journals/ws/ColucciDGS16,DBLP:conf/semweb/PetrovaSGH17,DBLP:conf/esws/HassadGJ17,amendola2026nexsim}.
Approximations for quasi-characterizations have also been investigated~\cite{DBLP:conf/semweb/PetrovaKGH19,DBLP:journals/data/ColucciDS24,DBLP:journals/access/ColucciDSSS25,amendola2026nexsim}.
}%
\textcolor{black}{In Database Theory, objects are tuples of constants and explanations are constant-free conjunctive queries, or unions of them.
Three questions have been studied.
The first asks for a query that returns all the given tuples and, among those queries, for a characterization~\cite{DBLP:conf/pods/CateDFL23}.
The second asks for a query that returns exactly the given tuples and no others; it is known as \emph{definability}~\cite{DBLP:conf/dagstuhl/Willard09,DBLP:conf/cp/Willard10,DBLP:journals/jsyml/BodirskyPT13,DBLP:conf/icdt/AntonopoulosNS13} and, when the given tuples form the output of an unknown query, as \emph{reverse engineering}~\cite{DBLP:conf/sigmod/ZhangEPS13,DBLP:journals/vldb/TranCP14}.
The third provides, in addition, a set of tuples to be excluded~\cite{DBLP:conf/icdt/CateD15,DBLP:conf/icdt/Barcelo017}.
The second and the third question have also been investigated in the presence of an ontology~\cite{DBLP:conf/cikm/CimaCL21,DBLP:conf/ijcai/Gutierrez-Basulto18}.}

\textcolor{black}{In this landscape, Amendola et al.~\cite{AMENDOLA2024120331} recently designed a logic-based framework to characterize the nexus of similarity within relational KBs, specifically tailored for practical deployment in large-scale real-world scenarios. (A first implementation has already been made public~\cite{amendola2026nexsim}.%
\footnote{\textcolor{black}{See \url{https://nexsim.it/}}})
Building upon the aforementioned existing approaches, adapting and recasting them, this framework is able to:}
$(i)$ accommodate different types of KBs;
 $(ii)$ adopt a notion of \emph{summary selector} to focus on relevant knowledge about the input items;
$(iii)$ handle, in input, sets of entity tuples;
$(iv)$ ensure that nexus explanations and nexus characterizations always exist, and admit concise and human understandable equivalent representations; and
$(v)$ extend the classical notion of linear expansions (i.e., $\unit \subset \unit' \subset \unit'' \subset \ldots$), to generalize (tuples of) entities in a taxonomic way.
More precisely, the framework is based on the notion of a  \emph{selective KB}, denoted by $\kb = (K,\varsigma)$, which enriches any knowledge base $K$ with a so-called \emph{summary selector} $\varsigma$. For each tuple $\tau$ of entities, $\varsigma$ identifies a relevant portion of the knowledge ``entailed'' by $K$ that meaningfully describes $\tau$.

To formally express commonalities in this setting, the framework introduces a dedicated \emph{nexus explanation language}, called $\WCQ$, equipped with suitable semantics. Within this language, one can define $\WCQ$-formulas playing the role of \emph{explanations} and \emph{characterizations} for tuples of entities, together with succinct forms of the latter called \emph{canonical characterizations} and \emph{core characterizations}. Importantly, all these formulas are guaranteed to exist and are effectively computable.
In addition, the framework defines an \emph{expansion graph}, which generalizes the classical notion of linear expansion by capturing taxonomic structures that naturally emerge from similarity relations.
These structures can be intuitively viewed as \emph{taxonomic entity set expansions}, reflecting how humans group and generalize concepts along meaningful semantic hierarchies.
Key reasoning tasks related to the computation of characterizations and expansions are also addressed, with results showing tractability under practical assumptions.%

\newcommand{\tmpframea}[1]{\colorbox{Cyan!5}{#1}}
\begin{figure*}[t!]
\centering
\begin{tikzpicture}[x=3.35 cm,y=1.46cm]
	\SetUpEdge[lw = 0.5pt, color = black, labeltext = black!75, labelcolor = white]

	\GraphInit[vstyle=Normal] %

	\tikzset{VertexStyle/.style = {
			shape = rectangle, draw, color=black!75,
			font=\footnotesize\sffamily,
			fill = LimeGreen!25}
	}

	\Vertex{Discovery Cove}
	\SO(Discovery Cove){Florida}
	\SOWE(Discovery Cove){US}
	\tikzset{VertexStyle/.style = {
			shape = rectangle, draw, color=black!75,
			font=\footnotesize\sffamily,
			fill = LimeGreen!25}%
}
\EA(Discovery Cove){theme park}
\NOEA(theme park){amusement park}

\tikzset{VertexStyle/.style = {
		shape = rectangle, draw, color=black!75,
		font=\footnotesize\sffamily,
		fill = LimeGreen!25}
}
\SOEA(Discovery Cove){Epcot}
\EA(amusement park){Prater}
\SO(Prater){Austria}
\WE(Discovery Cove){California}
\NOWE(Discovery Cove){Pacific Park}
\SO(amusement park){Gardaland}
\SO(Gardaland){Italy}
\SO(Austria){Leolandia}

\tikzset{EdgeStyle/.style={->}}
\tikzset{EdgeStyle/.append style={font=\sffamily\footnotesize}}

\Edge[label=located](Discovery Cove)(Florida)
\tikzset{EdgeStyle/.style={->, font=\sffamily\footnotesize}}
\Edge[label=partOf](Florida)(US)
\tikzset{EdgeStyle/.style={->, font=\sffamily\footnotesize}}
\Edge[label=isa](Discovery Cove)(theme park)
\tikzset{EdgeStyle/.style={->,  font=\sffamily\footnotesize}}
\Edge[label=isa](theme park)(amusement park)
\Edge[label=located](Epcot)(Florida)
\Edge[label=isa](Epcot)(theme park)
\tikzset{EdgeStyle/.style={->, font=\sffamily\footnotesize}}
\Edge[label=isa](Prater)(amusement park)
\Edge[label=located](Prater)(Austria)
\Edge[label=located](Pacific Park)(California)
\Edge[label=partOf](California)(US)
\Edge[label=isa](Pacific Park)(amusement park)
\Edge[label=located](Gardaland)(Italy)
\Edge[label=isa](Leolandia)(amusement park)
\Edge[label=located](Leolandia)(Italy)
\Edge[label=isa](Gardaland)(theme park)
\tikzset{EdgeStyle/.style={->, font=\sffamily\footnotesize}}
\tikzset{EdgeStyle/.append style={font=\sffamily\footnotesize}}
\tikzset{EdgeStyle/.style={->, font=\sffamily\footnotesize}}
\end{tikzpicture}%
\caption{Knowledge graph $\mathcal{G}_0$ underlying the selective knowledge base $\kb_0$ of our Example.}%
\label{fig:KG1}
\end{figure*}

\newcommand{\tmpframe}[1]{\colorbox{Yellow!10}{#1}}

\begin{figure*}[t!]
\centering
\begin{minipage}[c]{0.48\textwidth}
\centering
\begin{tikzpicture}[x=2.2cm, y=1.4cm]
	\SetUpEdge[lw = 0.5pt, color = black!75, labeltext = black!75, labelcolor = white]
	\GraphInit[vstyle=Normal]
	\tikzset{VertexStyle/.style = {
			shape = rectangle, draw, color=black!75,
			font=\scriptsize\sffamily, %
			fill = Lavender!20,
            inner sep=3pt}
	}

	\Vertex[Math,L=\!\bar{\varphi}_a\mapsto\{\langle\textsf{Disc.~Cove}\rangle{\tt,}\langle\textsf{Epcot}\rangle\}\,{\tt=}\,\unit_0]{Discovery Cove; Epcot}

	\NOWE[Math,L=\!\bar{\varphi}_b\mapsto\{\langle\textsf{Pacific~Park}\rangle\}\!](Discovery Cove; Epcot){Pacific Park}

	\NOEA[Math,L=\!\bar{\varphi}_c \mapsto\{\langle\textsf{Gardaland}\rangle\}\!](Discovery Cove; Epcot){Gardaland}

	\NOWE[Math,L=\!\bar{\varphi}_d\mapsto\{\langle\textsf{Prater}\rangle{\tt,} \langle\textsf{Leolandia}\rangle\}\!](Gardaland){Prater; Leolandia}

	\NO[Math,L=\!\bar{\varphi}_e\mapsto\{\langle\textsf{theme park}\rangle\}\!](Prater; Leolandia){theme park}

	\NO[Math,L=\!\bar{\varphi}_f \mapsto\{\dots\}\!](theme park){any other entity}

	\tikzset{EdgeStyle/.style={->}}
	\tikzset{EdgeStyle/.append style={font=\sffamily\scriptsize}}
	\Edge[label=](Discovery Cove; Epcot)(Pacific Park)
	\Edge[label=](Discovery Cove; Epcot)(Gardaland)
	\Edge[label=](Gardaland)(Prater; Leolandia)
	\Edge[label=](Pacific Park)(Prater; Leolandia)
	\Edge[label=](Prater; Leolandia)(theme park)
	\Edge[label=](theme park)(any other entity)
\end{tikzpicture}
\end{minipage}%
\hfill %
\begin{minipage}[c]{0.48\textwidth}
\vspace{-2.5mm}
\begin{equation*}\setlength{\jot}{5.0pt}\scriptsize %
	\begin{split}
		\bar{\varphi}_f = x \leftarrow & \ \top(x)\\
		\bar{\varphi}_e = x \leftarrow & \ \mathsf{isa}(x,\mathsf{ap}), \\
                                       & \ \top(x), \top(\mathsf{ap})\\
		\bar{\varphi}_d = x \leftarrow & \ \mathsf{isa}(x,\mathsf{ap}), \mathsf{located}(x,y), \\
                                       & \ \top(x), \top(y), \top(\mathsf{ap})\\
		\bar{\varphi}_c = x \leftarrow & \ \mathsf{isa}(x, \mathsf{tp}), \top(\mathsf{tp}), \\
                                       & \ \mathit{conj}(\bar{\varphi}_d)\\
		\bar{\varphi}_b = x \leftarrow & \ \mathsf{isa}(x,\mathsf{ap}), \mathsf{located}(x,y), \\
                                       & \ \top(x), \top(y), \top(\mathsf{ap}),\\
		                               & \ \mathsf{partOf}(y,\mathsf{US}), \top(\mathsf{US})\\
		\bar{\varphi}_a = x \leftarrow & \ \mathsf{isa}(x, \mathsf{tp}), \mathsf{isa}(x,\mathsf{ap}), \\
		                               & \ \mathsf{located}(x,\mathsf{Florida}), \\
                                       & \ \mathsf{partOf}(\mathsf{Florida},\mathsf{US}), \\
		                               & \ \top(x), \top(\mathsf{tp}), \top(\mathsf{ap}), \\
		                               & \ \top(\mathsf{Florida}), \top(\mathsf{US}),
	\end{split}
\end{equation*}
\end{minipage}
\caption{Expansion graph $\mathit{eg}(\unit_0,\kb_0)$, where $\mathsf{tp} = \mathsf{theme\_park}$, $\mathsf{ap} = \mathsf{amusement\_park}$, and $\mathit{conj}(\varphi)$ is the conjunction of atoms of $\varphi$.}\label{fig:KG3} \vspace{-3mm}
\end{figure*}

\subsection{Working Example}\label{sec:example}

To illustrate the key notions and reasoning tasks introduced above, let us focus on a simple example based on the knowledge graph $\mathcal{G}_0$ graphically depicted in Figure~\ref{fig:KG1}.
Now, given a set of entities---say  $\unit_0$ $=$ $\{\langle\mathsf{Discovery~Cove}\rangle$, $\langle\mathsf{Epcot}\rangle\}$, referred to as an {\em anonymous relation} or a {\em unit}---the goal is to comprehensively express the nexus of similarity between the elements of $\unit_0$ and unveil its expansions.

To achieve the goal, the framework relies on the notion of {\em selective knowledge base} that we are going to define on top of $\mathcal{G}_0$. This new object consists of a pair $\kb_0 = (K_0, \varsigma_0)$, where $K_0=(D_0,O_0)$ is a KB composed by a relational dataset $D_0$ encoding $\mathcal{G}_0$ paired with an ontology $O_0$  expressing some intentional knowledge about $D_0$, and $\varsigma_0$ is a an algorithm that selects, for each tuple of entities $\tau$ from $D_0$, a portion of knowledge ---called {\em summary} of $\tau$ --- from the one entailed by $K_0$, plus an atom of the form $\top(e)$ for each selected entity $e$.
Intuitively, the knowledge entailed by $K_0$, denoted by $\mathit{ent}(K_0)$, collects all the atoms the occur in every model of $K_0$, whereas the special predicate $\top$ simply states that $e$ is an entity.
Accordingly, the knowledge graph $\mathcal{G}_0$ can be naturally encoded as the dataset:

\begin{center}
$\begin{array}{l}
D_0 = \{\mathsf{isa}(\mathsf{Epcot},\mathsf{tp}),\
\mathsf{isa}(\mathsf{tp},\mathsf{ap}),\
\mathsf{located}(\mathsf{Epcot},\mathsf{Florida}),\
\mathsf{partOf}(\mathsf{Florida},\mathsf{US}),\
\ldots\}.
\end{array}$
\end{center}

\noindent Moreover, as intentional knowledge, we can
have a single Datalog~\cite{DBLP:books/aw/AbiteboulHV95} rule expressing the transitive closure of the binary predicate $\mathsf{isa}$, defined via the ontology:

\begin{center}
$\begin{array}{l}
O_0 = \{\mathsf{isa}(x,z) \leftarrow \mathsf{isa}(x,y), \ \mathsf{isa}(y,z)\}.
\end{array}$
\end{center}

\noindent Hence, the atoms entailed by $K_0 = (D_0, O_0)$ are:

\begin{center}
$\begin{array}{l}
\mathit{ent}(K_0) = D_0 \cup \left\{\mathsf{isa}(\mathsf{Discovery~Cove},\mathsf{ap}), \ \mathsf{isa}(\mathsf{Epcot},\mathsf{ap}), \ \mathsf{isa}(\mathsf{Gardaland},\mathsf{ap})\right\}.
\end{array}$
\end{center}

Consider now again the set $\unit_0$ $=$ $\{\langle\mathsf{Discovery~Cove}\rangle$, $\langle\mathsf{Epcot}\rangle\}$.
\noindent Concerning the summaries, instead, we can consider the simple yet effective selector $\varsigma_0$ that (when applied to unitary tuples as those in $\unit_0$) builds, for each $e$ in $D_0$, the dataset $\varsigma_0(\langle e \rangle)$ as the union of the following three sets:
\begin{itemize}
    \item $A(e)$ collecting all the atoms of $\mathit{ent}(K_0)$ having $e$ in the first position; essentially, these atoms represent the arcs of $\mathcal{G}_0$ outgoing from node $e$ plus those in the transitive closure of the $\mathsf{isa}$ relation of $\mathcal{G}_0$ still outgoing from $e$. For instance, $A(\mathsf{Epcot}) = \{\mathsf{isa}(\mathsf{Epcot},\mathsf{tp}),$ $
\mathsf{isa}(\mathsf{Epcot},\mathsf{ap}),
\mathsf{located}(\mathsf{Epcot},\mathsf{Florida})\}$.

    \item $B(e)$ collecting all the atoms of $\mathit{ent}(K_0)$ having some $e'$ in the first position, where $e'$ occurs in the second position of some atom $p(e,e')$ of $A(e)$ with $p \neq \mathsf{isa}$;  essentially, these atoms represent the arcs of $\mathcal{G}_0$ outgoing from nodes reachable from $e$ in one step. For instance, $ B(\mathsf{Epcot}) = \{\mathsf{partOf}(\mathsf{Florida},\mathsf{US})\}$.

    \item $C(e)$ collecting all the atoms of the form $\top(e')$, where
    $e'$ is an entity appearing in $A(e)$ or $B(e)$.
    For instance, $ C(\mathsf{Epcot}) = \{\top(\mathsf{Epcot}),
\top(\mathsf{tp}),
\top(\mathsf{ap}),
\top(\mathsf{Florida}),
\top(\mathsf{US})\}$.

\end{itemize}

We are now at the point in which we can try to express nexus of similarity between the tuples of $\unit_0$. To this aim, by taking into account the  considered summaries, let us first  examine the following formula:

\begin{center}
$\begin{array}{l}
    \bar{\varphi}_1 = x \leftarrow \mathsf{isa}(x, \mathsf{ap}), \mathsf{located}(x,y), \mathsf{partOf}(y,\mathsf{US}).
\end{array}$
\end{center}

\noindent It is evident that $\bar{\varphi}_1$  {\em explains} some nexus of similarity between $\langle\mathsf{Discovery~Cove}\rangle$ and $\langle\mathsf{Epcot}\rangle$.
However, $\bar{\varphi}_1$ neglects the additional information that both entities are also located in Florida  according to their summaries.
Conversely, it can be formally proved that, the formula $\bar{\varphi}_a$ in Figure~\ref{fig:KG3} fully explains the nexus of similarity between the two entities according to $\kb_0$.
Hence, $\bar{\varphi}_a$ characterizes the nexus of similarity between the tuple of $\unit_0$ within $\kb_0$.

The last step is to classify each entity $e$ of $\kb_0$ in relation to  $\unit_0$, by characterizing each unit $\unit_0 \cup \{\langle e \rangle\}$.
This leads to the {\em expansion graph} of $\unit_0$ with respect to $\kb_0$, denoted by $\mathit{eg}(\unit_0,\kb_0)$ as  depicted in Figure~\ref{fig:KG3}.
Intuitively, each node $n_1$ labeled by $\varphi_1 \mapsto \unit_1$ says that $\varphi_1$ characterizes $\unit_0 \cup \{\langle e \rangle\}$ for each $\langle e \rangle \in \unit_1$.
If there is a path from $n_1$ to another node $n_2$ labeled by $\varphi_2 \mapsto \unit_2$, it means that $\varphi_2$ characterizes the unit $\unit_0 \cup \unit_1 \cup \unit_2$ as well.
Thus, we can conclude, for instance, that
the  nexus of similarity that $\langle \mathsf{Gardaland}\rangle$ has with $\unit_0$
incorporate
those that $\langle \mathsf{Leolandia} \rangle$ has with $\unit_0$,  showing that $\mathsf{Gardaland}$ is more similar to the entities of $\unit_0$ than $\mathsf{Leolandia}$ with respect to $\kb_0$.
Additionally, the  nexus of similarity that $\langle \mathsf{Pacific~Park}\rangle$ has with $\unit_0$ are incomparable to those that $\langle \mathsf{Gardaland}\rangle$ has with $\unit_0$.
In simple terms, $\mathit{eg}(\unit_0,\kb_0)$ is the expected taxonomic expansion of $\unit_0$.

\subsection{Motivation}

The logic-based framework introduced by Amendola et al.~\cite{AMENDOLA2024120331} represents a foundational step toward a principled treatment of entity similarity and its representation through \emph{taxonomic entity set expansions} within knowledge bases. In particular, the introduction of the \emph{expansion graph} captures generalizations along taxonomic hierarchies, mirroring how humans group and relate entities by nexus of similarity.

However, such graphs can become extremely large in realistic scenarios, making it infeasible to materialize or explore them exhaustively. This calls for the development of efficient \emph{local navigation strategies}, enabling users to traverse the expansion graph incrementally, starting from an initial set of entities and proceeding step by step through semantically meaningful generalizations or refinements.

To support local navigation, the identification of \emph{navigation primitives} is essential. These primitives should capture key relationships between entities within the graph—such as belonging to the same node, lying in different but comparable nodes, or being located in entirely incomparable parts. At present, the framework does not provide such tools, creating a gap between its expressive theoretical potential and its practical usability.

\subsection{Contribution}

Our work addresses precisely the aforementioned navigational gaps. We propose and study novel computational tasks that serve as core navigation primitives to enable efficient and intuitive navigation within the expansion graph, thereby supporting the practical use of taxonomic entity set expansions. Intuitively, given an input unit~$\unit$ of entity tuples together with two extra tuples $\tau_1$ and $\tau_2$:

\begin{itemize}
  \item \textbf{SIM:} determine whether the expansions of $\unit \cup \{\tau_1\}$ and $\unit \cup \{\tau_2\}$ coincide. This corresponds to checking whether $\tau_1$ and $\tau_2$ occur in the same node of the expansion graph. For example, both $\langle \mathsf{Prater} \rangle$ and $\langle \mathsf{Leolandia} \rangle$ appear in the same node of $\mathit{eg}(\unit_0,\kb_0)$, implying the  formula associated to that node characterizes both $\unit_0 \cup \{\langle \mathsf{Prater} \rangle\}$ and $\unit_0 \cup \{\langle \mathsf{Leolandia} \rangle\}$.

  \item \textbf{PREC:} determine whether the characterization of $\unit \cup \{\tau_1\}$ is strictly more specific than that of $\unit \cup \{\tau_2\}$. In terms of the expansion graph, this amounts to verifying whether the node containing $\tau_2$ is reachable from the node containing $\tau_1$. For instance, the nexus of similarity that $\langle \mathsf{Gardaland} \rangle$ has with $\unit_0$ incorporate those that $\langle \mathsf{Leolandia} \rangle$ has with $\unit_0$, indicating that $\mathsf{Gardaland}$ lies in a more specific node than $\mathsf{Leolandia}$.

  \item \textbf{INC:}  determine whether the characterizations of $\unit \cup \{\tau_1\}$ and $\unit \cup \{\tau_2\}$ are mutually incomparable. This happens when the corresponding nodes in the expansion graph are not connected by any path in either direction. For example, the nexus of similarity that $\langle \mathsf{Pacific\ Park} \rangle$ has with $\unit_0$ are incomparable with those that $\langle \mathsf{Gardaland} \rangle$ has with $\unit_0$, meaning that these two tuples expand $\unit_0$ in semantically unrelated directions.
\end{itemize}

The goal of our work is to provide algorithms for our tasks that are asymptotically optimal.
To better classify the computational resources needed by these tasks when executed in different real-world scenarios, we exploit the theoretical assumptions already introduced in~\cite{AMENDOLA2024120331}, which rely on certain parameters that influence the size or the shape of all possible input units or selective knowledge bases at hand (see Section~\ref{sec:preliminaries}).

Essentially, under the \emph{broad} assumption, the maximum arity of predicates in the underlying knowledge base is bounded by a fixed constant (as is common in RDF triples, OWL ontologies, and Knowledge Graphs), while summaries are assumed to be computable in polynomial time.
Under the \emph{medium} assumption, also the cardinality of the input unit is bounded by a fixed constant, a situation frequently encountered in tasks such as semantic similarity and entity set expansion.
Finally, under the \emph{handy} assumption, we also require that the number of entities appearing in any summary is, roughly speaking, bounded by a logarithm of the size of the underlying knowledge base; this in a sense expresses the typical assumption behind the well-established task of entity summarization.

Let us now summarize our findings:
\textcolor{black}{$(i)$ in the broad setting, all tasks are intractable; indeed, \textsc{sim} is $\mathsf{NEXP}$-complete, \textsc{inc} is $\mathsf{coNEXP}$-complete, and \textsc{prec} is $\mathsf{DEXP}$-complete;}
$(ii)$ in the medium case, all the tasks are clearly intractable; indeed, although they are no harder than $\mathsf{DP}$~\footnote{We recall that $\mathsf{DP}$ is the closure under intersection of $\mathsf{NP} \cup \mathsf{coNP}$.}, under the standard computational belief that $\mathsf{P} \neq \mathsf{NP}$, they are not in
$\mathsf{P}$; and
$(iii)$ finally, under the handy assumption, all the three tasks are tractable; indeed they are in $\mathsf{P}$.

\textcolor{black}{Besides the three navigation tasks, we revisit the three tasks studied by Amendola et al.~\cite{AMENDOLA2024120331}, namely the computation of the canonical characterization (\textsc{can}), the computation of the core characterization (\textsc{core}) and membership in the essential expansion (\textsc{ess}): we improve three of their lower bounds, all in the broad case, and we complete the classification of these three tasks with the upper bounds that were left unstated.}

\textcolor{black}{Each of the three navigation tasks just introduced reduces, in the Turing sense, to two calls to an \textsc{ess} oracle; this allows us to replace the comparison between two characterizations whose sizes could, in general, be exponential.
The upper bounds do not follow from this reduction alone, which, as discussed in Section~\ref{sec:upper}, yields a larger class; they follow from the fact that the two calls are independent of each other and that their answers are combined in a way fixed in advance, as shown by Proposition~\ref{prop:membership} in Section~\ref{sec:tasks}.
The bounds for the three navigation tasks are tight, whereas a gap remains for the two functional ones, as discussed in Section~\ref{sec:conclusions}.}

\section{Preliminaries}\label{sec:basics}

\textcolor{black}{We now recall the classical background notions needed in the subsequent sections; for a comprehensive treatment, we refer the reader to~\cite{DBLP:books/aw/AbiteboulHV95,DBLP:books/daglib/0018514,DBLP:reference/fai/3}. At the same time, we fix the general notation used throughout the paper, aligning it with~\cite{AMENDOLA2024120331}.}

\subsection{Basics on Relational Structures}

We assume three pairwise disjoint, countably infinite
sets of symbols: $\PS$ for {\em predicates}, $\C$ for {\em constants}, and $\V$ for {\em variables}.
Variables are typically denoted by symbols such as $x$, $y$, $z$, and their indexed variants.
Constants and variables are collectively referred to as {\em terms}.
Each predicate $p$ has an associated {\em arity}, denoted by $|p|$, which indicates the number of arguments it takes.
An {\em atom} is an expression of the form $p(t_1,\ldots,t_n)$, where $p$ is a predicate of arity $n$, and $t_1,\ldots,t_n$ is a sequence of terms.
A {\em structure} $S$ is any set of atoms. Its {\em domain}, denoted by $\Dom_S$, is the set of all terms occurring in the atoms of $S$.
We assume that $\PS$ contains a distinguished unary predicate $\top$, referred to as {\em top}, which is used to explicitly indicate the presence of terms in $S$.
Intuitively, an atom $\top(a)$ asserts that the constant $a$ belongs to $\Dom_S$, independently of any specific relational role it may play.
The {\em closure under} predicate $\top$ of $S$ is the structure $S^{\top} = S \cup \{\top(t) \mid t \in \Dom_S\}$.
We say that $S$ is {\em closed under} $\top$ if $S = S^\top$.
For example, the structure $S = \{\mathsf{p}(1,2), \top(1)\}$ is not closed under $\top$, since $S^{\top} = S \cup \{\top(2)\}$ is different from $S$.
A structure $S$ is  {\em connected} if it consists of a single atom, or if it can be partitioned into two nonempty connected substructures whose domains share at least one term.
For example, the structure $S_1 = \{\mathsf{p}(a,b),~\mathsf{r}(a)\}$ is connected, while $S_2 = S_1 \cup \{\mathsf{r}(1)\}$ is not.
Indeed, $S_1$ can be partitioned into $S_1' = \{\mathsf{p}(a,b)\}$ and $S_1'' = \{\mathsf{r}(a)\}$, which are both trivially connected and share the constant $a$.

Consider two structures $S$ and $S'$, together with a set $\mathit{T}$ of terms.
A {\em $\mathit{T}$-homomorphism} from $S$ to $S'$ is a function $h: \Dom_S \rightarrow \Dom_{S'}$ satisfying the following conditions: $(i)$ $h(t) = t$ for each $t \in T \cap \Dom_S$; and
$(ii)$ for every atom $p(t_1,\ldots,t_n)$ in $S$, the atom $p(h(t_1),\ldots,h(t_n))$ belongs to $S'$.
We may write a $\mathit{T}$-homomorphism explicitly as a set of mappings of the form $t \mapsto t'$.
For an atom $\alpha = p(t_1,\ldots,t_n)$, we write $h(\alpha)$ to denote the atom $p(h(t_1),\ldots,h(t_n))$.
If $h$ is bijective, then it is called a {\em $\mathit{T}$-isomorphism}, denoted $S \simeq_{\mathit{T}} S'$.
When $T = \emptyset$, we simply refer to homomorphisms and isomorphisms, omitting the $\mathit{T}$ prefix.

An $n$-ary {\em tuple} is an expression of the form $\langle t_1,\ldots,t_n\rangle$, where each $t_i$ is a term.
Intuitively, a tuple can be seen as an ``anonymous'' atom, or conversely, an atom can be viewed as a ``labeled'' tuple.
For a tuple $\tau$ (resp., a sequence $\mathbf{s}$) of terms, its $i$-th term is denoted by $\tau[i]$ (resp., $\mathbf{s}[i]$).
Given a set $T$ of terms, $T^n$ denotes the set of all $n$-ary tuples that can be formed with elements of $T$.
The {\em domain} of a set $J$ of tuples is the set $\Dom_J$ of terms appearing in the tuples of $J$.

\subsection{Knowledge Bases and Conjunctive Formulas}
A {\em knowledge base} (KB) is a formal representation of knowledge, typically organized into two distinct components: a {\em dataset} $D$, which collects factual extensional assertions about a domain, and an {\em ontological theory} $O$, which encodes intensional knowledge through general rules, constraints, or background information in a logical form.
Depending on the setting, $D$ can take the form of a relational database, an ABox, or the factual part of a knowledge graph. Similarly, $O$ may correspond to a TBox, a set of logical dependencies (such as tuple-generating dependencies or Datalog rules), or any rule-based theory expressed in a fragment of first-order logic.
Formally, in the rest of the paper, a knowledge base is a pair $K = (D, O)$, where the dataset $D$ is a finite, nonempty structure over constants, and $O$ is an ontological theory.
\textcolor{black}{In general terms, a \emph{model} of a knowledge base $K = (D, O)$ is a (possibly infinite) structure over constants that, typically, contains all facts in $D$ and in addition satisfies the theory $O$. Accordingly,}
an atom $\alpha$ is said to be \emph{entailed} by a knowledge base $K$, written $K \models \alpha$, if $\alpha$ belongs to every model of $K$.
We denote by $\mathit{ent}(K)$ the set of all atoms entailed by $K$.

A {\em conjunctive formula} (CF, for short) is an expression $\varphi$ of the form
\begin{equation}\label{eq:formula}
x_1,\ldots,x_n \leftarrow p_1({\bf t}_1),\ldots, p_m({\bf t}_m),
\end{equation}
\noindent where $n \geq 0$ is the {\em arity} of $\varphi$, $m > 0$ is its {\em size} (denoted by $|\varphi|$), each $p_i({\bf t}_i)$ is an atom formed using a predicate $p_i$ and a sequence of terms ${\bf t}_i$, and each $x_j$ is a variable ---called {\em free}--- occurring in at least one atom of $\varphi$.
The sequence of atoms $p_1({\bf t}_1), \ldots, p_m({\bf t}_m)$ is referred to as the body of $\varphi$ and denoted by $\mathit{conj}(\varphi)$; the set of such atoms is denoted by $\mathit{atm}(\varphi)$.
\textcolor{black}{Essentially, $\varphi$ is a {\em primitive-positive formula}~\cite{DBLP:journals/jacm/Rossman08} without equality, extended with constants.}
\textcolor{black}{Syntactically, $\varphi$ corresponds to a standard {\em conjunctive query} (CQ)~\cite{DBLP:books/aw/AbiteboulHV95} in which constants are allowed.
However, since such expressions are used to explain and characterize commonalities among entities rather than to retrieve data, we avoid the operational connotations traditionally associated with ``queries'' and adopt the term {\em conjunctive formula}, following~\cite{AMENDOLA2024120331}.
This choice of terminology also directly connects with foundational works on definability by Ross Willard~\cite{DBLP:conf/dagstuhl/Willard09, DBLP:conf/cp/Willard10}, who pioneered the study of the complexity of deciding whether a given set of tuples can be characterized by a {\em primitive-positive formula} (i.e., a constant-free CQ with equality) that captures exclusively that set, as well as classic results on homomorphism preservation~\cite{DBLP:journals/jacm/Rossman08}.}
We say that $\varphi$ is {\em open} if $n \geq 1$, i.e., it contains at least one free variable. Moreover, $\varphi$ is {\em connected} if the structure $\mathit{atm}(\varphi)$ is connected.
\textcolor{black}{As for CQs,} the {\em output} to $\varphi$ over a dataset $D$ is the set $\varphi(D)$ of all tuples $\langle t_1, \ldots, t_n \rangle$ for which there exists a $\C$-homomorphism from $\mathit{atm}(\varphi)$ to $D$ that maps each $x_i$ to $t_i$.
Consider, for example, the \textcolor{black}{conjunctive} formula $\bar{\varphi} = x \leftarrow \mathsf{p}(x,y),~\mathsf{q}(y,z),~\mathsf{r}(z)$. According to Equation~\ref{eq:formula}, its arity is $n = 1$ (one free variable $x$), and its size is $m = 3$ (three atoms). Hence, $\bar{\varphi}$ is both open and connected.
Let $D$ be the dataset $\{\mathsf{p}(a,b),~\mathsf{q}(b,c),~\mathsf{r}(c),~\mathsf{p}(d,e),~\mathsf{q}(e,f),~\mathsf{r}(f)\}$. Then, the output $\bar{\varphi}(D)$ contains the tuple $\langle a \rangle$, witnessed by the homomorphism $\{x \mapsto a, y \mapsto b, z \mapsto c\}$. Similarly, $\langle d \rangle$ is included via the homomorphism $\{x \mapsto d$, $y \mapsto e$, $z \mapsto f\}$. These are the only tuples in $\bar{\varphi}(D)$.

Consider now two \textcolor{black}{conjunctive} formulas $\varphi$ and $\varphi'$ with free variables $x_1,\ldots,x_n$ and $y_1,\ldots,y_n$, respectively.
We say that $\varphi$ can be {\em homomorphically mapped} to $\varphi'$ if there exists a $\C$-homomorphism $h$ from $\mathit{atm}(\varphi)$ to $\mathit{atm}(\varphi')$ that maps each $x_i$ to $y_i$. In this case, we write $\varphi \longrightarrow \varphi'$. Otherwise, we write $\varphi \longarrownot\longrightarrow \varphi'$.
In particular, if $h$ is bijective, then $\varphi$ and $\varphi'$ are said to be {\em isomorphic}, and we write $\varphi \simeq \varphi'$.
Moreover, if $\varphi \longrightarrow \varphi'$ and $\varphi' \longarrownot\longrightarrow \varphi$, then $\varphi$ is a {\em generalization} of $\varphi'$; if both $\varphi \longarrownot\longrightarrow \varphi'$ and $\varphi' \longarrownot\longrightarrow \varphi$, then $\varphi$ and $\varphi'$ are {\em incomparable}; and if both $\varphi \longrightarrow \varphi'$ and $\varphi' \longrightarrow \varphi$, denoted by $\varphi \longleftrightarrow \varphi'$, then the formulas are {\em equivalent}.

\subsection{Computational Complexity}

We assume the reader is familiar with the notions of \emph{Turing machine} and \emph{reduction}~\cite{DBLP:books/daglib/0018514}.
In what follows, we adopt standard complexity classes to determine the computational cost of our decision and functional tasks.
For decision problems, we consider the hierarchy:
$\mathsf{NL} \subseteq \mathsf{P} \subseteq \mathsf{NP}$ and $\mathsf{coNP} \subseteq \mathsf{DP} \subseteq \mathsf{PH} \subseteq \mathsf{PSPACE} = \mathsf{NPSPACE} \subseteq \mathsf{EXP} \subseteq \mathsf{NEXP}$ and $\mathsf{coNEXP} \subseteq \mathsf{DEXP}$, where $\mathsf{DP}$ (resp., $\mathsf{DEXP}$) is the closure of $\mathsf{NP} ~\cup~ \mathsf{coNP}$ (resp., $\mathsf{NEXP} ~\cup~ \mathsf{coNEXP}$) under intersection.
For functional problems, we refer to the classes:
$F\mathsf{P} = F\mathsf{P}^{\mathsf{NL}} \subseteq F\mathsf{P}^{\mathsf{NP}} \subseteq F\mathsf{P}^{\mathsf{PH}} \subset F\mathsf{PSPACE} = F\mathsf{PSPACE}^{\mathsf{NPSPACE}} \subseteq F\mathsf{EXP}^{\mathsf{NP}}$, where $F\mathsf{PSPACE}$ denotes the class of functions computable by a deterministic Turing machine equipped with a one-way, write-only output tape and a constant number of work tapes of polynomial length.

\section{Framework Overview}\label{sec:preliminaries}

\textcolor{black}{This section briefly overviews the framework proposed by Amendola et al.~\cite{AMENDOLA2024120331}, which was informally introduced in Section~\ref{sec:intro}.
We remark that, while some definitions draw upon long-standing notions established across various fields, we adopt the terminology of Amendola et al., tailored for selective KBs and $\WCQ$ formulas.}

\subsection{\textcolor{black}{Units and Selective Knowledge Bases}}
Given any $n > 0$, an $n$-ary {\em anonymous relation} ---referred to simply as a {\em unit}--- is a finite, nonempty set of $n$-ary tuples of constants that need to be characterized or expanded.
To characterize the tuples of a unit, the framework introduces the notion of a {\em summary selector}.
In essence, it is an algorithm, say $\varsigma$, that takes as input a knowledge base $K$ together with a tuple $\tau$ of constants, and returns a dataset $\varsigma(K,\tau)$ ---the {\em summary} of $\tau$ (with respect to $K$)---
collecting atoms from $\mathit{ent}(K)^\top$ that are relevant to describing $\tau$.
Importantly, $\varsigma(K,\tau)$ is closed under $\top$, and its domain includes all the constants of $\tau$; \textcolor{black}{without these conditions, the existence of explanations and characterizations would not be guaranteed.}
When the context is clear, we may simply write $\varsigma(\tau)$ in place of $\varsigma(K, \tau)$.
Concrete summary selectors can be defined by adapting techniques from~\cite{stickler2005cbd,DBLP:journals/ws/ColucciDGS16,DBLP:journals/ws/LiuCGQ21,DBLP:journals/semweb/Pirro19}.
A {\em selective knowledge base} ---abbreviated as SKB--- is defined as a pair $\kb = (K, \varsigma)$, where $K$ is a KB and $\varsigma$ is a summary selector.

\textcolor{black}{To appreciate the advantages of adopting SKBs, consider the alternative design choices when characterizing an $n$-ary unit $\unit = \{\tau_1, \ldots, \tau_m\}$.
The most direct approach would be to operate over a single ``global'' dataset $D$. However, constructing a canonical characterization in this setting yields conjunctive formulas of size $\mathcal{O}(|\Dom_D|^m)$, which in real-world scenarios translates to billions of atoms, rendering characterizations both computationally unmanageable and humanly uninterpretable.
To mitigate this explosion, one might instead associate a ``local'' dataset $D_i$ with each input tuple $\tau_i \in \unit$.
While such static assignment appears sufficient for characterizing $\unit$ in isolation, it immediately breaks down when navigating the expansion graph by comparing tuples from $\Dom_D^n$ to those in $\unit$, as every candidate tuple would require its own dedicated local dataset.
Rather than pairing each tuple with an ad-hoc  dataset, relying on a uniform mapping is far more natural: this is precisely the role of the summary selector, which systematically yields local datasets from a global one. This design choice preserves consistency with the framework of Amendola et al.~\cite{AMENDOLA2024120331} while aligning with established Semantic Web practices, such as using rooted RDF-graphs as summaries~\cite{DBLP:journals/ws/ColucciDGS16}.
Note that requiring summaries to be closed under $\top$ is crucial. Consider the unit $\unit = \{\langle\mathsf{a},\mathsf{b}\rangle,\langle\mathsf{c},\mathsf{d}\rangle\}$. Assume that
\[\begin{array}{l}
\varsigma(\langle\mathsf{a},\mathsf{b}\rangle) = \{\mathsf{Professor}(\mathsf{a}), \, \top(\mathsf{a}), \, \mathsf{Professor}(\mathsf{b}), \, \top(\mathsf{b})\} \quad \mbox{and}\\[2pt]
\varsigma(\langle\mathsf{c},\mathsf{d}\rangle) = \{\mathsf{Professor}(\mathsf{c}), \, \top(\mathsf{c}), \, \mathsf{Student}(\mathsf{d}), \, \top(\mathsf{d})\}.
\end{array}\]
Accordingly, the conjunctive formula $x, y \leftarrow \mathsf{Professor}(x), \, \top(x), \, \top(y)$ is a characterization of $\unit$, which would fail to exist without $\top$-closure.
Finally, enriching the underlying global dataset with rules emerges as a natural and practically advantageous extension, completing the transition to a full KB from which summaries can be extracted.}

\subsection{\textcolor{black}{The Explanation Language}}
We can now introduce the explanation language used to characterize anonymous relations.
Consider an open \textcolor{black}{conjunctive} formula $\varphi$ (i.e., \textcolor{black}{one with at least a free variable}) as in Equation~\ref{eq:formula}.
An atom $\alpha$ is {\em connected} to a free variable $x$ of $\varphi$ if there is a connected structure $S$, consisting of atoms from $\varphi$, that includes both $\alpha$ and $x$.
Building on this notion, we say that $\varphi$ is
{\em  nearly connected} if each atom in $\varphi$ is connected to at least one of its free variables.
In other words, $\varphi$ is nearly connected if the structure $\mathit{atm}(\varphi) \cup \{\mathsf{free}(x_1, \ldots, x_n)\}$ is connected, where $\mathsf{free}(x_1, \ldots, x_n)$ is a fresh, dummy atom that syntactically collects all the free variables of $\varphi$.
This condition ensures that each atom in $\varphi$ is (directly or indirectly) linked to at least one atom that mentions a free variable.
Clearly, every connected open \textcolor{black}{CF} is also nearly connected; however, the converse does not necessarily hold.
Nearly connected conjunctive formulas, abbreviated as $\WCQ$, form the explanation language adopted in this framework.
\textcolor{black}{This language offers a good expressivity trade-off for capturing nexus of similarity, extending both rooted RDF-graphs~\cite{DBLP:journals/ws/ColucciDGS16} (which are essentially unary $\WCQ$s) and rooted CQs~\cite{DBLP:journals/ai/JungLPW22} (which coincide with nearly connected $\WCQ$s without constants).
On the one hand, incorporating constants, existential quantification, multiple free variables, and joins enables the explanation of rich, interconnected properties shared by entities.
In particular, a binary formula like $x, y \leftarrow \mathsf{Coauthor}(x, z), \, \mathsf{Coauthor}(y, z), \, \mathsf{Coauthor}(z, \mathsf{Erd}\text{\H{o}}\mathsf{s})$ captures a shared pattern between $x$ and $y$ anchored to a specific individual: an essential feature, as explanations lacking constants would be significantly less informative.
On the other hand, excluding disjunction, negation, universal quantification, and built-in equality avoids reasoning beyond local summaries and prevents subtle semantic anomalies. For instance, allowing disjunction would lead to uninformative explanations that merely collect disjoint facts about each input tuple without identifying any shared property; analogously, permitting built-in equality ($=$) would enable bindings to specific constants (e.g., $x = \mathsf{a}$), failing to preserve structural symmetries between entities.
Crucially, restricting attention to formulas that are nearly connected prevents the inclusion of noisy background assertions. Indeed, an atom disconnected from all free variables provides no shared interconnected property about the entities being compared.
For example, consider the unit $\unit = \{\mathsf{a}, \mathsf{b}\}$ such that
$\varsigma(\mathsf{a}) = \{\mathsf{Person}(\mathsf{a}), \, \mathsf{friend}(\mathsf{a},\mathsf{c}), \, \mathsf{Person}(\mathsf{c}), \, \top(\mathsf{a}), \, \top(\mathsf{c})\}$ and
$\varsigma(\mathsf{b}) = \{\mathsf{Person}(\mathsf{b}), \, \mathsf{parent}(\mathsf{b},\mathsf{c}), \, \mathsf{Person}(\mathsf{c}), \, \top(\mathsf{b}), \, \top(\mathsf{c})\}$.
Accordingly, a characterization for $\unit$ would be $x \leftarrow \mathsf{Person}(x), \, \mathsf{Person}(\mathsf{c}), \top(x), \top(\mathsf{c})$, where the atoms $\mathsf{Person}(\mathsf{c})$ and $\top(\mathsf{c})$, holding across both summaries, merely assert  isolated background facts about the KB.
Finally, by avoiding rigid restrictions such as single connected components or acyclicity, $\WCQ$ formulas guarantee the existence of finite characterizations that would otherwise fail to exist (see Section 7.2 in Amendola et al.~\cite{AMENDOLA2024120331} for complete technical details).}

The semantics of a $\WCQ$ formula $\varphi$ is defined in terms of the tuples it captures with respect to a selective knowledge base $\kb = (K, \varsigma)$; these tuples are called the instances of $\varphi$ according to $\kb$.
Formally, a tuple $\tau$ is an {\em instance} of $\varphi$ if it belongs to the output of $\varphi$ when evaluated over its summary $\varsigma(\tau)$.
The set of all such tuples is denoted by $\mathit{inst}(\varphi,\kb)$.
Importantly, $\mathit{inst}(\varphi, \kb)$ contains at most the tuples in the output of $\varphi$ over the full dataset $\mathit{ent}(K)^\top$, since each $\varsigma(\tau)$ may include only partial information.
\textcolor{black}{Formally,  $\mathit{inst}(\varphi, \kb) \subseteq \varphi(\mathit{ent}(K)^\top)$.}

Consider an SKB $\kb = (K, \varsigma)$ with $K = (D, O)$ and an $n$-ary unit $\unit = \{\tau_1, \ldots, \tau_m\}$.
An $\WCQ$ formula $\varphi$ \textit{explains} the nexus of similarity among the tuples of $\unit$ if  $\mathit{inst}(\varphi, \kb) \supseteq \unit$.
Accordingly, we may also say that $\varphi$ explains $\unit$, that $\varphi$ is an explanation for $\unit$, or that $\unit$ is explained by $\varphi$ with respect to $\kb$.
Furthermore, $\varphi$ \textit{characterizes} the nexus of similarity among the tuples in $\unit$ if it explains $\unit$ and, for every \textcolor{black}{\WCQ} formula $\varphi'$ that also explains $\unit$, it holds that $\varphi' \longrightarrow \varphi$.
In this case, we also say that $\varphi$ characterizes $\unit$, that $\varphi$ is a characterization for $\unit$, or that $\unit$ is characterized by $\varphi$ with respect to $\kb$.
Intuitively, since all characterizations of a unit are \textcolor{black}{$\C$}-homomorphically equivalent, they provide a complete explanation of the nexus of similarity among the tuples in the unit.
\textcolor{black}{Although the definition of explanation is based on the notion of an instance of a $\WCQ$ formula with respect to an SKB, the underlying rationale of a characterization is the same as that adopted in least general generalizations~\cite{plotkin1971further}, most specific concepts~\cite{DBLP:conf/aaai/CohenBH92}, and most specific similarity queries~\cite{DBLP:conf/semweb/PetrovaSGH17}.}
Importantly, in our case, every unit admits explanations and characterizations of finite size.

\textcolor{black}{It is worth emphasizing that the evaluation mechanism, in which each tuple $\tau$ is evaluated against its own summary $\varsigma(\tau)$, represents another key architectural and semantic design choice.
Consider a unit $\unit$, an SKB $\kb = (K, \varsigma)$, and a CF $\varphi$.
Had we instead defined $\mathit{inst}(\varphi, \kb)$ globally as $\varphi(\mathit{ent}(K)^\top)$—evaluating $\varphi$ uniformly over the full dataset $\mathit{ent}(K)^\top$—it could happen that $\varphi$ explains $\unit$ under this broader notion of instances (i.e., $\mathit{inst}(\varphi, \kb) \supseteq \unit$), yet for some tuple $\tau \in \unit$ there is no $\C$-homomorphism from $\mathit{atm}(\varphi)$ to $\varsigma(\tau)$.
In such a case, $\varphi$ would fail to capture genuine, relevant commonalities, as it would assert nexus of similarity derived from facts outside the summaries of $\unit$. Effectively, the resulting explanations would degenerate to those obtained by setting the summary of every tuple in $\unit$ to $\mathit{ent}(K)^\top$. Consequently, characterizations would completely bypass the summary mechanism, since the most specific explanations would then be forced to account for the entire dataset $\mathit{ent}(K)^\top$.}

\subsection{\textcolor{black}{Characterizations and Essential Expansions}}
Amendola et al.~\cite{AMENDOLA2024120331} have shown that each unit $\unit$ admits a {\em canonical characterization}, denoted by $\mathit{can}(\unit,\kb)$, which can be effectively computed and has size at most exponential in the size of $\unit$ and of the summaries of its tuples (see Section~\ref{sec:computing} for details).
\textcolor{black}{This was already implicit for constant-free CQs without summaries~\cite{DBLP:conf/icdt/CateD15}, but it was unknown for the specific setting adopted in this paper; indeed, the adaptation of the techniques used for CQs ---devised by Amendola et al.~\cite{AMENDOLA2024120331} and exploiting the notion of {\em clones} (see Section~\ref{sec:computing})--- was necessary and non-trivial.}
\textcolor{black}{Near-connectedness has a similar counterpart in the construction: among the atoms that the summaries of the tuples of $\unit$ have in common, those that are not connected to any free variable have to be discarded. In particular, such an atom adds no shared property, as observed above, and yet keeping it would make a tuple fail to be an instance whenever its own summary does not contain that atom, thus ruling out tuples that do share the properties the formula is meant to express. Section~\ref{sec:computing} makes both steps explicit.}
There are cases, however, in which the aforementioned bound on the size of $\mathit{can}(\unit,\kb)$ is not optimal.
To obtain characterizations that are as succinct as possible, \textcolor{black}{one can rely on the well-known notion of {\em core} of a structure~\cite{DBLP:journals/tods/FaginKP05} to define the {\em core characterization} of a unit}: a $\mathit{core}$ of a \textcolor{black}{CF} $\varphi$ is any \textcolor{black}{CF} $\varphi'$ such that
$\mathit{atm}(\varphi') \subseteq \mathit{atm}(\varphi)$,
$\varphi \longrightarrow \varphi'$,
and there exists no \textcolor{black}{CF}  $\varphi''$ with $\mathit{atm}(\varphi'') \subset \mathit{atm}(\varphi')$ such that $\varphi' \longrightarrow \varphi''$.
In other words, a core of $\varphi$ is a minimal (with respect to set inclusion on atoms) \textcolor{black}{CF} such that the output of $\varphi'$ over any dataset coincides with the output of $\varphi$ over that same dataset.
Clearly, any \textcolor{black}{conjunctive} formula isomorphic to $\varphi'$ is also  a core of $\varphi$.
As usual, we refer to the core of $\varphi$ and denote it by $\mathit{core}(\varphi)$.
Accordingly, we denote by $\mathit{core}(\unit, \kb)$ the core characterization of the unit $\unit$ with respect to the SKB $\kb$.

Building on these notions, we introduce the essential expansion of $\unit$, which includes all tuples that, when added to $\unit$, yield an extended unit with the same overall characterization.
Formally, the {\em essential expansion} of $\unit$, denoted by $\mathit{ess}(\unit, \kb)$, is the set $\mathit{inst}(\varphi, \kb)$, where $\varphi = \mathit{core}(\unit, \kb)$.
\textcolor{black}{This notion of expansion can be viewed as the analogue of the {\em pp-closure} defined by Willard~\cite{DBLP:conf/dagstuhl/Willard09} for primitive positive formulas in a summary-free setting.}

\begin{proposition}[\cite{AMENDOLA2024120331}]\label{prop:resfortask1}
    Let $\tau_1$ and $\tau_2$ be two $n$-ary tuples. It holds that $\mathit{ess}(\unit\cup\{\tau_1\},\kb)\subseteq\mathit{ess}(\unit\cup\{\tau_2\},\kb)$ if, and only if, $\mathit{core}(\unit\cup\{\tau_2\},\kb)\longrightarrow\mathit{core}(\unit\cup\{\tau_1\},\kb)$.
\end{proposition}

By construction, $\mathit{ess}(\unit, \kb) \supseteq \unit$ always holds, and every tuple in $\mathit{ess}(\unit, \kb) \setminus \unit$ preserves the same nexus of similarity when considered together with the original tuples in $\unit$.
\subsection{\textcolor{black}{Expansion Graphs}}
Having isolated the tuples that are maximally similar to those in $\unit$, namely all those in $\mathit{ess}(\unit, \kb) \setminus \unit$, we now turn to classifying all other tuples in relation to $\unit$.
This leads to the construction of the \emph{taxonomic expansion} of $\unit$, which organizes tuples based on how the core characterization of $\unit$ evolves when extended with each additional tuple.
Let $\mathbf{T}$ denote the set of all $n$-ary tuples over the domain of $\kb$.
For each $\tau \in \mathbf{T}$, the formula $\mathit{core}(\unit \cup \{\tau\}, \kb)$ captures the nexus of similarity between $\tau$ and the original unit $\unit$.
Informally, this formula corresponds to the most concise least general generalization of $\mathit{core}(\unit, \kb)$ that explains all tuples in $\unit \cup \{\tau\}$.
Crucially, the similarity relation between tuples can be captured syntactically via $\C$-homomorphisms between such core characterizations:
given two tuples $\tau, \tau' \in \mathbf{T}$, the existence of a \textcolor{black}{$\C$}-homomorphism from $\mathit{core}(\unit \cup \{\tau'\}, \kb)$ to $\mathit{core}(\unit \cup \{\tau\}, \kb)$ indicates that the nexus of similarity that $\tau'$ shares with $\unit$ are at most those that $\tau$ shares with $\unit$;
if the two core characterizations are \textcolor{black}{$\C$}-homomorphically equivalent, then $\tau$ and $\tau'$ share exactly the same nexus of similarity with $\unit$.
These observations motivate the following definition, which introduces the expansion graph as a structure that organizes tuples into a hierarchy based on the \textcolor{black}{$\C$}-homomorphism relation between their induced core \textcolor{black}{conjunctive} formulas.

\begin{definition}[\cite{AMENDOLA2024120331}]\label{def-exp-graph}
The {\em expansion graph} of $\unit$ according to $\kb$, denoted by $\mathit{eg}(\unit,\kb)$, is the triple $(N,A,\delta)$ where:
\begin{itemize}
\item[$(i)$] $N$ is the set of nodes, each representing the core characterization of a unit $\unit \cup \{\tau\}$ for some $\tau \in \mathbf{T}$, where units with isomorphic cores are identified with the same node;
\item[$(ii)$] $A$ is a set collecting each pair $(\varphi_1,\varphi_2)$ of distinct nodes of $N$ such that $\varphi_2 \longrightarrow \varphi_1$ and there is no other (i.e., different from $\varphi_1$ and $\varphi_2$) \textcolor{black}{conjunctive} formula $\varphi_3 \in N$ such that $\varphi_2 \longrightarrow \varphi_3$ and $\varphi_3 \longrightarrow \varphi_1$; and
\item[$(iii)$] $\delta$ is a function \textcolor{black}{assigning} to each node $\varphi \in N$ the set of tuples
$\mathit{inst}(\varphi,\kb) \setminus \{\tau \in \mathit{inst}(\varphi',\kb):(\varphi',\varphi) \in A\}$,
called the set of {\em direct instances} of $\varphi$, and denoted by $\mathit{dinst}(\varphi, \kb)$. \mybox
\end{itemize}
\end{definition}

\begin{table}[t!]
	\centering%
	\begin{tabular}{ccc}
		\hline
		\hspace{0.8cm}{\bf  Problem}\hspace{0.8cm} & \hspace{1.0cm}{\bf Input}\hspace{1.0cm} & \hspace{1.5cm}{\bf Reasoning Task}\hspace{1.5cm}   \\
		\hline
	\textsc{can}  & $\kb, \unit$  & compute $\mathit{can}(\unit,\kb)$ \\
	\textsc{core}  & $\kb, \unit$  & compute $\mathit{core}(\unit,\kb)$\\
        \hline
			\textsc{ess} & $\kb, \unit, \tau$ & does $\tau \in \mathit{ess}(\unit,\kb)$ hold?\\
  \hline
	\end{tabular}\normalsize%
 \caption[Old tasks.]{Tasks defined in~\cite{AMENDOLA2024120331}; everywhere $\{\tau\} \cap \unit = \emptyset$.}%
 \label{tab:Oldtasks}
\end{table}

\begin{table}[t!]
\centering
\begin{tabular}{cccc}
		\hline
		~~~~~\hspace{.2cm}{\bf  Problem } & \hspace{.1cm}{\bf Broad} & {\bf Medium}  & \hspace{1cm}{\bf Handy}\hspace{1cm} \\
		\hline

		~~~~~\textsc{can}  & ~~~~~ $F\mathsf{P}^{\mathsf{NP}}$-hard ~~~~~ & in $F\mathsf{P}$ & ~~~~~in $F\mathsf{P}$~~~~~ \\

  ~~~~~\textsc{core}  & $F\mathsf{P}^{\mathsf{NP}}$-hard   & in $F\mathsf{P}^{\mathsf{NP}}$ & in $F\mathsf{P}$\\

  \hline
		~~~~~\textsc{ess} &  $\mathsf{NP}$-hard   & $\mathsf{NP}$-hard  & in $\mathsf{P}$  \\
  \hline
	\end{tabular}
 \caption[Computational complexity of tasks.]{%
 \textcolor{black}{Computational complexity results established in~\cite{AMENDOLA2024120331} under the three different theoretical assumptions summarized at the end of Section~\ref{sec:preliminaries}.}}
	\label{tab:complexityOld}\normalsize
\end{table}

The expansion graph $\mathit{eg}(\unit, \kb)$ thus organizes the tuples in $\mathbf{T}$ according to the core characterizations of all extensions of $\unit$.
It is directed and acyclic, since edges are defined via \textcolor{black}{$\C$}-homomorphisms between non-isomorphic cores.
The sets of direct instances at each node form a partition of $\mathbf{T}$.
The node associated with $\mathit{core}(\unit,\kb)$ is the unique source of the graph, and its direct instances coincide with the essential expansion of $\unit$. See Figures~\ref{fig:KG1} and \ref{fig:KG3} in the Introduction for an example of expansion graph starting from a unary unit of size 2.

\subsection{\textcolor{black}{Reasoning Tasks and Computational Assumptions}}
We now turn to the reasoning tasks associated with this framework.
These tasks formalize key computational goals introduced by Amendola et al.~\cite{AMENDOLA2024120331}, which center around computing characterizations, and testing membership in the essential expansion.
Table~\ref{tab:Oldtasks} provides an overview of these tasks: \textsc{can} and \textsc{core} refer to the computation of $\mathit{can}(\unit, \kb)$ and $\mathit{core}(\unit, \kb)$, respectively, while \textsc{ess} refers to checking whether a tuple $\tau$ belongs to $\mathit{ess}(\unit, \kb)$.

To better analyze the computational resources required by the aforementioned tasks in different real-world settings, Amendola et al.~considered three theoretical assumptions, each reflecting constraints on certain parameters that influence the size and shape of input units or selective knowledge bases.
\textcolor{black}{Throughout, $\omega$ denotes the maximum arity of the predicates occurring in $K$, and all three assumptions require $\omega$ to be bounded. This requirement was already adopted by Amendola et al.~\cite{AMENDOLA2024120331} in view of practical deployment~\cite{amendola2026nexsim}, and it follows the long-standing practice in database theory of measuring complexity with respect to selected parameters of the input, which goes back to the distinction between data, expression and combined complexity~\cite{DBLP:conf/stoc/Vardi82}. In practice, $\omega$ is a parameter of the schema rather than of the data, and schemas are typically fixed in advance and often consist of predicates of small arity; in knowledge graphs, for instance, facts are triples, that is, binary atoms~\cite{DBLP:journals/csur/HoganBCdMGKGNNN21}. The requirement does not trivialize the problems under consideration: hardness results for closely related problems on conjunctive queries already hold over a single binary relation~\cite{DBLP:conf/icdt/CateD15,DBLP:conf/pods/CateDFL23}, and all the lower bounds of this paper likewise hold for bounded $\omega$, which is used only in the upper bounds (see Section~\ref{sec:overview}). The assumptions, listed from the least to the most restrictive, are as follows:}
\begin{itemize}
    \item[1.] $\mathit{broad}$, where the maximum arity $\omega$ of  predicates in $K$ is bounded by some fixed integer ($\omega$ is typically  ``small'' and even limited to two  in cases $\kb$ is derived  from a KG) and both $\varsigma(K,\tau)$ and $\mathit{ent}(K)$ are polynomial-time computable;%
    \footnote{\textcolor{black}{Polynomial-time computability ensures that the computational burden of summary extraction does not overshadow the core reasoning tasks.}}
\item[2.] $\mathit{medium}$, where, in addition, $m = |\unit| $ is bounded (this value is typically two in semantic similarity, and typically referred to as ``small'' in entity set expansion); and
\item[3.] $\mathit{handy}$, where, in addition, $|\Dom_{\varsigma(K,\tau)}| \in \mathcal{O}\left(\sqrt[m+1]{\log_2|\Dom_{\mathit{ent}(K)}|}\right)$ (the number of constants of a summary is typically referred to as ``small'' in entity summarization).
\end{itemize}

\noindent According to the above assumptions, Table~\ref{tab:complexityOld} summarizes the results already presented in~\cite{AMENDOLA2024120331}.
\textcolor{black}{The present paper completes that picture, and extends it to the three tasks introduced next.}

\section{Navigating Taxonomic Expansions of Entity Sets}
\label{sec:navigating-similarity}

\subsection{\textcolor{black}{Navigation Tasks}}
The expansion graph introduced in the previous section serves as a ``map of similarity'', where (tuples of) entities from the knowledge base are positioned with respect to the initial unit.
For instance, according to Figure~\ref{fig:KG1} and Figure~\ref{fig:KG3}, the tuple $\langle \mathsf{Gardaland} \rangle$ shares more nexus of similarity with $\unit_0$ than $\langle \mathsf{Leolandia} \rangle$, whereas the nexus of similarity between $\langle \mathsf{Pacific~Park} \rangle$ and the entities of $\unit_0$ are of a different and incomparable nature compared to those between $\langle \mathsf{Gardaland} \rangle$ and $\unit_0$. Indeed, $\langle \mathsf{Pacific~Park} \rangle$ is not a theme park, while $\langle \mathsf{Gardaland} \rangle$ is not located in the US.

Constructing the entire expansion graph, however, may be computationally expensive or prohibitive when operating with large-scale, real-world knowledge bases.
An alternative approach consists in navigating the graph locally by comparing pairs of tuples based on the nodes of which they are direct instances.
As shown in Figure~\ref{fig:KG3}, the tuples $\langle \mathsf{Gardaland} \rangle$ and $\langle \mathsf{Pacific~Park} \rangle$ are direct instances of two incomparable nodes, while $\langle \mathsf{Leolandia} \rangle$ and $\langle \mathsf{Prater} \rangle$ are direct instances of the same node.
This motivates some  fundamental questions regarding a taxonomic expansion of an input unit: \emph{Should one entity be presented before another because it shares more nexus of similarity to the input?} \emph{Can we identify entities that share the same nexus of similarity with respect to the input?} \emph{How can we determine whether entities share incomparable nexus of similarity with the input?}
To provide concrete answers to these questions, the most natural approach is to formalize them as reasoning tasks.
This is done in Table~\ref{tab:NewTasks}, where we rely on condition~$(ii)$ of Definition~\ref{def-exp-graph}, which specifies how nodes in an expansion graph are connected via arcs.
In simple terms, given an SKB $\kb$, an initial $n$-ary unit $\unit$, and two additional, distinct $n$-ary tuples $\tau$ and $\tau'$, we ask whether the nexus of similarity that $\tau$ shares with $\unit$:
\begin{itemize}
     \item are \emph{higher than} those that $\tau'$ shares with $\unit$ --- this corresponds to task \textsc{prec}, where the node whose direct instances include $\tau'$ is reachable from the node whose direct instances include $\tau$;

    \item \emph{coincides with} those that $\tau'$ shares with $\unit$ --- this corresponds to task \textsc{sim}, where there exists a node whose direct instances include both $\tau$ and $\tau'$.

    \item are \emph{incomparable to} those that $\tau'$ shares with $\unit$ --- this corresponds to task \textsc{inc}, where  the nodes whose direct instances include $\tau$ and $\tau'$ are mutually unreachable.
\end{itemize}

\begin{table}[t!]
\centering
\renewcommand{\arraystretch}{1.4}
\begin{tabular}{>{\centering\arraybackslash}m{2cm} >{\centering\arraybackslash}m{2cm} >{\centering\arraybackslash}m{7.5cm}>
{\centering\arraybackslash}m{2cm}}
    \hline
    {\bf Problem} & {\bf Input} & {\bf Reasoning Task} & {\bf Shorthand}\\
    \hline

    \multirow{2}{*}{\textsc{prec}} & \multirow{2}{*}{$\kb, \unit, \tau, \tau'$} &
    do $\mathit{core}(\unit \cup \{\tau'\}, \kb) \longrightarrow \mathit{core}(\unit \cup \{\tau\}, \kb)$ and & \multirow{2}{*}{$\tau \prec_{^\unit}^{_\kb} \tau'$}\\
    & &  $\mathit{core}(\unit \cup \{\tau\}, \kb) \longarrownot\longrightarrow \mathit{core}(\unit \cup \{\tau'\}, \kb)$ both hold?  \\

    \hline

    \multirow{2}{*}{\textsc{sim}} & \multirow{2}{*}{$\kb, \unit, \tau, \tau'$} &
    do $\mathit{core}(\unit \cup \{\tau'\}, \kb) \longrightarrow \mathit{core}(\unit \cup \{\tau\}, \kb)$ and & \multirow{2}{*}{$\tau \sim_{^\unit}^{_\kb} \tau'$}\\
    & &  $\mathit{core}(\unit \cup \{\tau\}, \kb) \longrightarrow \mathit{core}(\unit \cup \{\tau'\}, \kb)$ both hold?  \\

    \hline

    \multirow{2}{*}{\textsc{inc}} & \multirow{2}{*}{$\kb, \unit, \tau, \tau'$} &
    do $\mathit{core}(\unit \cup \{\tau'\}, \kb) \longarrownot\longrightarrow \mathit{core}(\unit \cup \{\tau\}, \kb)$ and & \multirow{2}{*}{$\tau \parallel_{^\unit}^{_\kb} \tau'$}\\
    & &  $\mathit{core}(\unit \cup \{\tau\}, \kb) \longarrownot\longrightarrow \mathit{core}(\unit \cup \{\tau'\}, \kb)$ both hold?\\

    \hline
\end{tabular}
\normalsize
\caption{New computational Tasks; everywhere $\{\tau\}\neq\{\tau'\}, \{\tau\}\cap \unit = \emptyset=\{\tau'\}\cap \unit$.}
\label{tab:NewTasks}
\end{table}

\noindent \textcolor{black}{The three tasks can also be read on the expansion graph. An affirmative answer to \textsc{sim} tells us that $\tau$ and $\tau'$ are direct instances of the same node of the expansion graph, and therefore that no explanation of $\unit$ distinguishes the two tuples. An affirmative answer to \textsc{prec} tells us that the node whose direct instances include $\tau'$ is reachable from the node whose direct instances include $\tau$, and therefore that $\tau$ shares with $\unit$ every nexus of similarity that $\tau'$ shares with $\unit$, and at least one more. An affirmative answer to \textsc{inc} tells us that the nodes whose direct instances include $\tau$ and $\tau'$ are mutually unreachable, and therefore that each of the two tuples shares with $\unit$ a nexus of similarity that the other one does not share.}

These new tasks are inherently based on the notions of core characterization and \textcolor{black}{$\C$}-homomorphism between \textcolor{black}{conjunctive} formulas.
Accordingly, their implementation would involve three main steps: first, constructing $\mathit{can}(\unit \cup \{\tau\})$ and $\mathit{can}(\unit \cup \{\tau'\})$; second, computing from them the corresponding core characterizations $\mathit{core}(\unit \cup \{\tau\})$ and $\mathit{core}(\unit \cup \{\tau'\})$; and finally, verifying the existence or non-existence of \textcolor{black}{$\C$}-homomorphisms from $\mathit{core}(\unit \cup \{\tau'\})$ to $\mathit{core}(\unit \cup \{\tau\})$, and vice versa.
This process, however, can be computationally expensive, primarily due to the construction of core characterizations.
\textcolor{black}{Indeed, in the medium column of Table~\ref{tab:complexityOld}, \textsc{can} is in $F\mathsf{P}$ whereas \textsc{core} is only known to be in $F\mathsf{P}^{\mathsf{NP}}$, and Theorem~\ref{thm:corecore} shows that \textsc{core} is not in $F\mathsf{P}$ unless $\mathsf{NP}=\mathsf{coNP}$.}
One might consider avoiding the construction of core characterizations altogether by performing \textcolor{black}{$\C$}-homomorphism checks directly on the canonical characterizations.
Yet, since canonical characterizations can be considerably large in practice, such checks may also become computationally burdensome.
\subsection{\textcolor{black}{Reformulation via Essential Expansions}}
In what follows, to mitigate these potential limitations, we reformulate our reasoning tasks in terms of essential expansions.

This reformulation is formally captured by the following proposition.

\begin{proposition}\label{taskViaEss}
The following hold:
\begin{itemize}
    \item $\tau \prec_{^\unit}^{_\kb} \tau'$ if, and only if, $\mathit{ess}(\unit \cup \{\tau\}, \kb) \subset \mathit{ess}(\unit \cup \{\tau'\}, \kb)$;

  \item $\tau \sim_{^\unit}^{_\kb} \tau'$ if, and only if,
    $\mathit{ess}(\unit \cup \{\tau\}, \kb) = \mathit{ess}(\unit \cup \{\tau'\}, \kb)$;

  \item $\tau \parallel_{^\unit}^{_\kb} \tau'$ if, and only if,  $\mathit{ess}(\unit \cup \{\tau\}, \kb) \setminus \mathit{ess}(\unit \cup \{\tau'\}, \kb) \neq \emptyset$ and
    $\mathit{ess}(\unit \cup \{\tau'\}, \kb) \setminus \mathit{ess}(\unit \cup \{\tau\}, \kb) \neq \emptyset$.
\end{itemize}
\end{proposition}

\begin{proof} By definition, according to Table~\ref{tab:NewTasks},
$\tau \prec_{^\unit}^{_\kb} \tau'$ holds if, and only if, both $\mathit{core}(\unit \cup \{\tau'\}, \kb) \longrightarrow \mathit{core}(\unit \cup \{\tau\}, \kb)$ and $\mathit{core}(\unit \cup \{\tau\}, \kb) \longarrownot\longrightarrow \mathit{core}(\unit \cup \{\tau'\}, \kb)$ hold.
Then, by Proposition~\ref{prop:resfortask1}, these conditions are equivalent to having both $\mathit{ess}(\unit \cup \{\tau\}, \kb) \subseteq \mathit{ess}(\unit \cup \{\tau'\}, \kb)$ and  $\mathit{ess}(\unit \cup \{\tau'\}, \kb) \not\subseteq \mathit{ess}(\unit \cup \{\tau\}, \kb)$.
This, in turn, is equivalent to the strict inclusion $\mathit{ess}(\unit \cup \{\tau\}, \kb) \subset \mathit{ess}(\unit \cup \{\tau'\}, \kb)$, thus proving the first statement.
The remaining two statements follow analogously by applying the same reasoning pattern and using Proposition~\ref{prop:resfortask1}.
\end{proof}

A naive implementation of our reasoning tasks based on Proposition~\ref{taskViaEss} is still not optimal.
Indeed, one would first need to compute $\mathit{can}(\unit \cup \{\tau\})$ and $\mathit{can}(\unit \cup \{\tau'\})$, then evaluate them over $\mathit{ent}(K)^\top$ to obtain the corresponding essential expansions $\mathit{ess}(\unit \cup \{\tau\}, \kb)$ and $\mathit{ess}(\unit \cup \{\tau'\}, \kb)$, and finally compare these sets of instances to determine their relationship.
Stated in this way, however, the tasks may appear even more demanding than in their original formulation.

Nevertheless, the reformulation via essential expansions enables a further step: it provides a foundation for simplifying and refining the actual checks that need to be performed.
This refinement is formally captured by the following proposition.

\begin{proposition}\label{taskViaEssEasy}
The following hold:
\begin{itemize}
    \item $\tau \prec_{^\unit}^{_\kb} \tau'$ if, and only if, $\tau \in \mathit{ess}(\unit \cup \{\tau'\}, \kb)$ and $\tau' \not\in \mathit{ess}(\unit \cup \{\tau\}, \kb)$;

  \item $\tau \sim_{^\unit}^{_\kb} \tau'$ if, and only if,
    $\tau \in \mathit{ess}(\unit \cup \{\tau'\}, \kb)$ and $\tau' \in \mathit{ess}(\unit \cup \{\tau\}, \kb)$;

  \item $\tau \parallel_{^\unit}^{_\kb} \tau'$ if, and only if,  $\tau \not\in \mathit{ess}(\unit \cup \{\tau'\}, \kb)$ and $\tau' \not\in \mathit{ess}(\unit \cup \{\tau\}, \kb)$;
\end{itemize}
\end{proposition}

\begin{proof} Below we report the proof of the first point, the other two points can be obtained following the same proof scheme.
 $\Longrightarrow$ By Proposition~\ref{taskViaEss}, the latter means that $\mathit{ess}(\unit \cup \{\tau\}, \kb) \subset \mathit{ess}(\unit \cup \{\tau'\}, \kb)$. Hence, since $\unit \cup \{\tau\}\subseteq \mathit{ess}(\unit \cup \{\tau\}, \kb) $, always holds, we get that $\tau \in \mathit{ess}(\unit \cup \{\tau'\}, \kb)$. We then need to prove that $\tau' \not\in \mathit{ess}(\unit \cup \{\tau\}, \kb)$. Suppose by contradiction that the latter does not hold. This implies that $\mathit{core}(\unit\cup\{\tau\},\kb)$ explains $\unit \cup \{\tau'\}$. By definition of characterization we then have that $\mathit{core}(\unit \cup \{\tau\}, \kb) \longrightarrow \mathit{core}(\unit \cup \{\tau'\},\kb)$ which is against our hypotheses.

 $\Longleftarrow$ Suppose $\tau \in \mathit{ess}(\unit \cup \{\tau'\}, \kb)$ and $\tau' \not\in \mathit{ess}(\unit \cup \{\tau\}, \kb)$. The latter means that $\mathit{core}(\unit\cup\{\tau'\},\kb)$ is an explanation for  $\unit\cup\{\tau\}$, and therefore $\mathit{core}(\unit \cup \{\tau'\}, \kb) \longrightarrow \mathit{core}(\unit \cup \{\tau\},\kb)$. Now suppose by contradiction that $\mathit{core}(\unit \cup \{\tau\}, \kb) \longrightarrow \mathit{core}(\unit \cup \{\tau'\})$, the above implies $\tau' \in \mathit{ess}(\unit \cup \{\tau\}, \kb)$ which is against our hypotheses.
   \end{proof}

\textcolor{black}{Proposition~\ref{taskViaEssEasy} establishes that membership for our navigational reasoning tasks can be determined by at most two invocations of \textsc{ess}, thereby providing a clear path for upper bound analysis.}
In particular, deciding whether $\tau$ belongs to $\mathit{ess}(\unit \cup \{\tau'\}, \kb)$, \textcolor{black}{as well as its converse}, no longer requires computing or comparing entire sets of instances.
Instead, \textcolor{black}{the problem reduces to constructing} the canonical characterizations $\mathit{can}(\unit \cup \{\tau'\})$ and $\mathit{can}(\unit \cup \{\tau\})$, and then \textcolor{black}{evaluating} them over the summaries of $\tau$ and $\tau'$, respectively.
\textcolor{black}{Specifically, the decision step verifies} whether $\tau'$ is in the output to $\mathit{can}(\unit \cup \{\tau\})$ over the summary of $\tau'$, and whether $\tau$ is in the output to $\mathit{can}(\unit \cup \{\tau'\})$ over the summary of $\tau$.

\textcolor{black}{Proposition~\ref{taskViaEssEasy} is a Turing reduction of each of \textsc{prec}, \textsc{sim} and \textsc{inc} to \textsc{ess}.
An algorithm for any of the three tasks calls a procedure for \textsc{ess} twice: the first call receives the input $(\kb, \unit \cup \{\tau'\}, \tau)$, and the second call receives the input $(\kb, \unit \cup \{\tau\}, \tau')$.
The two calls are independent of each other, because neither input depends on the answer to the other call.
The answer to the task is then a combination of the two answers, and that combination is fixed in advance: a conjunction for \textsc{sim}, a conjunction of the two negations for \textsc{inc}, and a conjunction of the first answer with the negation of the second one for \textsc{prec}.}

\textcolor{black}{What this reduction gives, and what it does not give, will be clarified in Section~\ref{sec:upper}.}

\section{\textcolor{black}{Overview of the computational results}}\label{sec:overview}

\textcolor{black}{A brief summary of our computational results is presented in Table~\ref{tab:complexityNew1}. In this section, we examine the table item by item, illustrating the origin of each result and paying particular attention to relating them to existing literature in terms of the techniques employed.}

\subsection{Organization of the results}
\textcolor{black}{In Table~\ref{tab:complexityNew1} there are eighteen entries, one for each pair of a task and an assumption, and it presents twenty-nine results overall.
Eleven entries state both a membership and a hardness result, while the six entries of the handy case and the entry of \textsc{can} in the medium case state a membership only.}

\textcolor{black}{Six results are due to Amendola et al.~\cite{AMENDOLA2024120331}: the membership of \textsc{can} and of \textsc{core} in the medium and in the handy case, the membership of \textsc{ess} in the handy case, and the $\mathsf{NP}$-hardness of \textsc{ess} in the medium case.
The other twenty-three results are presented here for the first time. Twenty of them concern a pair of a task and an assumption for which no bound was known. Three of them replace a bound of Amendola et al. by a stronger one: \textsc{can} and \textsc{core} were known to be $F\mathsf{P}^{\mathsf{NP}}$-hard in the broad case, and we show that neither function can be computed in $F\mathsf{P}^{\mathsf{PH}}$, while \textsc{ess} was known to be $\mathsf{NP}$-hard in the broad case, and we show that it is $\mathsf{NEXP}$-hard.
None of these three could be derived directly from the arguments given in Amendola et al.}

\textcolor{black}{
The results on \textsc{can} and on \textsc{core} are proved in Theorems~\ref{thm:CAN-RES} and~\ref{thm:corecore}, those on \textsc{ess}, \textsc{sim}, \textsc{prec} and \textsc{inc} in Theorems~\ref{thm:essfinal}, \ref{thm:sim}, \ref{thm:prec} and~\ref{thm:inc}, that is, in Section~\ref{sec:ess} and in Section~\ref{sec:tasks}, and the three memberships in Proposition~\ref{prop:membership}.}

\begin{table}[t!]
\centering
\begin{tabular}{cccc}
		\hline
		~~~~~{\bf  Problem } & {\bf Broad} & {\bf Medium}  & {\bf Handy} \\
		\hline
  ~~~~~\textsc{core}  & in $F\mathsf{EXP}^{\mathsf{NP}} \setminus F\mathsf{P}^{\mathsf{PH}}$  & in $F\mathsf{P}^{\mathsf{NP}} \setminus F\mathsf{P}^{\,\dagger}$ & in $F\mathsf{P}$\\

		~~~~~\textsc{can}  & ~~~~~in $F\mathsf{PSPACE} \setminus F\mathsf{P}^{\mathsf{PH}}$~~~~~ & in $F\mathsf{P}$ & ~~~~~in $F\mathsf{P}$~~~~~ \\

  \hline
		~~~~~\textsc{ess} &  $\mathsf{NEXP}$-complete   & $\mathsf{NP}$-complete  & in $\mathsf{P}$  \\
		~~~~~\textsc{prec} &   $\mathsf{DEXP}$-complete  &  $\mathsf{DP}$-complete   &  in $\mathsf{P}$  \\
		~~~~~\textsc{sim} &  $\mathsf{NEXP}$-complete   & $\mathsf{NP}$-complete   &  in $\mathsf{P}$  \\
		~~~~~\textsc{inc} &  $\mathsf{coNEXP}$-complete   & $\mathsf{coNP}$-complete   &  in $\mathsf{P}$  	 \\ \hline
	\end{tabular}
 \caption{Computational complexity of our decision and functional tasks under the three theoretical assumptions presented in Section~\ref{sec:preliminaries}. The lower bound marked by $\dagger$ holds under the assumption that $\mathsf{NP}\neq\mathsf{coNP}$.}
	\label{tab:complexityNew1}\normalsize
\end{table}

\textcolor{black}{Note that, in the tractable cases, results are presented as membership bounds only rather than completeness results. This choice stems from the fact that our framework relies on summaries computable in polynomial time, which inherently bounds the overall task even when the underlying reasoning is simpler. While these results are de facto tight in ontological settings where fact entailment is known to be $\mathsf{P}$-complete~\cite{DBLP:books/aw/AbiteboulHV95}, claiming full completeness would obscure the distinction between the complexity of entity comparison and that of summary computation. A fine-grained complexity analysis ---spanning among established ontological settings (with entailment ranging from $\mathsf{AC}^0$ to $\mathsf{P}$) as well as investigating structural constraints on summaries (e.g., assuming acyclic datasets) that could lower complexity in currently tractable and intractable cases--- is out of the scope of this paper and is left for future work.}

\textcolor{black}{Table~\ref{tab:complexityNew1} is a deliberately simplified presentation of what we prove: the full statements carry parameters that would make the picture unreadable.
For each $k>0$ and each $\pi \in \{\textsc{ess},\textsc{sim},\textsc{prec},\textsc{inc}\}$, write $\pi_k$ for the variant of $\pi$ in which the input unit is $k$-ary.
Two strengthenings are not recorded in the table. First, every hardness result of the table for these four tasks holds already for $\pi_k$, for each fixed $k>0$, with the only exception of \textsc{prec}, whose construction needs $k>1$; the unary case of \textsc{prec} is open, as discussed in Section~\ref{sec:conclusions}. Second, the hardness results hold in general already for a selective knowledge base with an empty ontology, and in the medium case already for a unit of two tuples, or of four for \textsc{prec}.}

\subsection{Upper bounds}\label{sec:upper}

\textcolor{black}{By Proposition~\ref{taskViaEssEasy}, each of \textsc{prec}, \textsc{sim} and \textsc{inc} is decided by two membership tests in an essential expansion, and each of the two tests is an instance of \textsc{ess}.}
\textcolor{black}{Since summaries have at most polynomial size even under the broad assumption, they are generally exponentially more succinct than canonical characterizations. Hence, applying Proposition~\ref{taskViaEssEasy} provides a structural advantage over checking for a \textcolor{black}{$\C$}-homomorphism between formulas that may themselves be of exponential size.}

\textcolor{black}{However, a Turing reduction to \textsc{ess} places the three tasks in a class that is larger than the three classes we claim for them.
In the medium case \textsc{ess} is in $\mathsf{NP}$, so the three tasks belong to $\mathsf{P}^{\mathsf{NP}[2]}$, the class of the languages that a deterministic polynomial-time machine decides by querying an oracle in $\mathsf{NP}$ at most twice.
That class contains $\mathsf{NP}$, $\mathsf{coNP}$ and $\mathsf{DP}$ altogether, and it is itself contained in $\Theta^{\mathsf{P}}_2 = \mathsf{P}^{\mathsf{NP}[\log]}$; in the broad case the same argument gives $\mathsf{EXP}^{\mathsf{NEXP}[2]}$.
The membership of \textsc{sim} in $\mathsf{NP}$, of \textsc{inc} in $\mathsf{coNP}$ and of \textsc{prec} in $\mathsf{DP}$ therefore does not follow from the reduction.
It follows from two further facts: the calls are exactly two and independent of one another, and the Boolean combination of their answers is fixed before the input is read. Proposition~\ref{prop:membership} turns these two facts into the three upper bounds.}

\textcolor{black}{The reduction to two membership tests matters for a second reason: computing the two essential expansions explicitly, and then comparing them, would not keep the tasks within these classes.}
\textcolor{black}{Specifically, already in the medium case where the size of the canonical characterization is polynomial, both $\mathit{ess}(\unit \cup \{\tau\}, \kb)$ and $\mathit{ess}(\unit \cup \{\tau'\}, \kb)$ may be of exponential size, as the number of candidate tuples depends on the arity of the input unit, which is not bounded. Hence, one should avoid this approach.}

\textcolor{black}{The upper bound for \textsc{ess} is Algorithm~\ref{alg:InEss}, which builds the canonical characterization of the unit and then looks for a $\C$-homomorphism from that formula to the summary of the tuple under test. In the broad case that formula has exponential size and the search is nondeterministic, which gives $\mathsf{NEXP}$; in the medium case the formula has polynomial size, which gives $\mathsf{NP}$; in the handy case the summaries are small enough for all candidate mappings to be enumerated, which gives $\mathsf{P}$.}

\textcolor{black}{The upper bounds for \textsc{can} and for \textsc{core} have a different origin: the algorithms are those of Amendola et al.~\cite{AMENDOLA2024120331}, and what we add is the analysis of their cost under each assumption, the algorithm for \textsc{core} being the classical minimization of a conjunctive query~\cite{chandra1977optimal}.}

\textcolor{black}{All the upper bounds above rely on the bound on the maximum arity $\omega$ of the predicates, which enters through the cloned atoms of Line~4 of Algorithm~\ref{alg:connectedCan} as a factor exponential in $\omega$; none of the lower bounds of the following sections uses it.}

\subsection{Lower bounds: strategies}\label{sec:strategies}

\textcolor{black}{
The four decision tasks of Table~\ref{tab:complexityNew1} are complete for six different classes, and the three navigation tasks inherit the hardness of \textsc{ess} through three different mechanisms.
For \textsc{sim} we reduce \textsc{ess} to \textsc{sim}.
For \textsc{inc} a reduction from \textsc{ess} is of no use, because an accepted instance of \textsc{inc} requires two negative answers, and we therefore reduce \textsc{ess} to the complement of \textsc{inc}.
For \textsc{prec} the hardness of \textsc{ess} is not sufficient at all. A language of $\mathsf{DP}$ is the intersection of a language of $\mathsf{NP}$ and of a language of $\mathsf{coNP}$, so no reduction from a single problem can force the two answers that an instance of \textsc{prec} requires; we start from an arbitrary language of the class and we build the two instances ourselves.
The three constructions are moreover constrained from two sides at once.
On one side, each of them has to return a selective knowledge base, hence a summary selector that is defined on every tuple of the domain and computable in polynomial time, and each of them has to control which constants occur in the characterizations of the unit it builds, since a $\C$-homomorphism cannot generalize a constant away.
On the other side, none of them may raise the number of tuples of the unit beyond a constant factor, nor the arity of the predicates. 
Each construction is given for a unary unit, and the devices of Section~\ref{sec:gadgets} carry it to $\pi_k$ for every fixed $k$, with the only exception of \textsc{prec}, which needs $k>1$; the same section states the facts on graph colorings that the medium case uses.}

\subsection{Lower bounds: techniques}\label{sec:techniques}

\textcolor{black}{The eleven hardness results of Table~\ref{tab:complexityNew1} rest on seven constructions.
Five of them are polynomial-time reductions: one for \textsc{ess} under each of the broad and the medium assumption, and one for each of \textsc{sim}, \textsc{prec} and \textsc{inc}, which applies unchanged under both assumptions and therefore yields two results.
The sixth is the family of units and selective knowledge bases of Corollary~\ref{cor:exponential}, which yields the lower bounds for \textsc{can} and \textsc{core} under the broad assumption.
The seventh is a Turing reduction, which yields the lower bound for \textsc{core} under the medium assumption.
We present the seven constructions in increasing order of difficulty.
The first is the reduction for \textsc{sim}.
Essentially, it adds to the selective knowledge base a copy of one tuple of the unit: every atom of the dataset in which the constant of that tuple occurs is copied, with that constant renamed into a new one, and the copy receives as summary the summary of the original tuple under the same renaming.
The copy therefore belongs to the essential expansion of the unit, and the given tuple satisfies \textsc{sim} together with the copy if, and only if, it belongs to that essential expansion, that is, if, and only if, the given instance of \textsc{ess} is accepted.
The second is the family of Corollary~\ref{cor:exponential}.
The summary of the $i$-th tuple of the unit is a directed cycle whose length is the $i$-th prime number.
Since these lengths are pairwise coprime, the direct product of the summaries is a single cycle, whose length is the product of the lengths: it is connected, so the canonical characterization keeps all of it, and it is a core, so the core characterization cannot shrink it.
Both characterizations therefore have a size that is exponential in the number of tuples of the unit.
The third is the reduction for \textsc{ess} under the medium assumption.
Here the unit has a bounded number of tuples, so the direct product of their summaries has polynomial size and the known results on products yield no lower bound; we reduce instead from the $3$-colorability of a graph~\cite{Kar72}.
Essentially, the dataset contains two disjoint copies of the given graph, each extended with a new vertex adjacent to all its vertices (Definition~\ref{def:complement}), together with a complete graph on four vertices.
The unit consists of the two tuples formed by the new vertices, which keeps its characterization free of constants of the dataset, and the given tuple is formed by a vertex of the complete graph.
The characterization of the unit is then equivalent to the extended graph, with the new vertex as its free variable, and the given tuple is an instance of it if, and only if, the extended graph is $4$-colorable, that is, by Lemma~\ref{lem:color1}, if, and only if, the given graph is $3$-colorable.
This lower bound was first established by Amendola et al.~\cite{AMENDOLA2024120331}.
We give a construction of our own for two reasons: by Lemma~\ref{lem:twins} it yields the hardness of $\textsc{ess}_k$ for every fixed $k$, and it produces instances with an empty ontology and a unit of at least two tuples, which is the form from which the reductions for \textsc{sim}, \textsc{prec} and \textsc{inc} start.
The fourth is the reduction for \textsc{ess} under the broad assumption.
We start from the problem of deciding whether the direct product of a sequence of instances admits a homomorphism into a further instance, which Willard~\cite{DBLP:conf/cp/Willard10} proved to be $\mathsf{NEXP}$-complete, and which ten Cate and Dalmau~\cite{DBLP:conf/icdt/CateD15} proved to remain so over a single binary relation.
Essentially, the dataset contains the given instances, each extended with a new constant linked to all of its elements; the unit consists of the tuples formed by the new constants of the instances in the product, and the given tuple is formed by the new constant of the further instance.
The characterization of the unit then encodes the direct product, and the given tuple is an instance of it if, and only if, the product admits a homomorphism into the further instance.
The known reduction does not apply as it stands, and two adjustments make this correspondence exact in our setting: the new constants make every extended instance connected, so that near-connectedness discards no part of the product; and one more instance, consisting of a single self-loop, is added to the product, which changes no homomorphism but guarantees that no constant of the dataset occurs in the characterization.
The fifth is the Turing reduction for \textsc{core} under the medium assumption.
It starts from core identification, which asks, given two finite structures one included in the other, whether the smaller one is isomorphic to the core of the larger one, and which Fagin et al.~\cite{DBLP:journals/tods/FaginKP05} proved to be $\mathsf{DP}$-complete.
Essentially, the two structures are turned into two constant-free nearly connected formulas closed under $\top$, in a way that preserves cores, and the larger formula is turned into a selective knowledge base with a unit of two tuples whose canonical characterization is that formula; a single call to \textsc{core} then returns the core of that formula, and a test in $\mathsf{NP}$ checks whether the smaller formula is isomorphic to it.
If \textsc{core} were in $F\mathsf{P}$ under the medium assumption, core identification would therefore be in $\mathsf{NP}$, and $\mathsf{NP} = \mathsf{coNP}$ would follow.
The step that builds a selective knowledge base from a given formula goes in the direction opposite to the task, which computes a formula from a given selective knowledge base, and it shows that bounding the number of tuples of a unit does not restrict the shape of its characterizations.
The sixth is the reduction for \textsc{inc}, which has to force two negative answers at the same time: neither of the two tuples may share with the unit everything the other one shares, as happens for $\langle \mathsf{Pacific~Park}\rangle$ and $\langle \mathsf{Gardaland}\rangle$ in the working example.
Essentially, it adds a copy of one tuple of the unit, as for \textsc{sim}, and it marks with a new unary predicate every constant of the original dataset but not the copy.
The copy then lacks exactly one property shared by the unit, namely the mark, while the given tuple, whenever it does not belong to the essential expansion of the unit, lacks some other property.
The two tuples are therefore incomparable if, and only if, the given tuple does not belong to the essential expansion of the unit, so that the construction reduces \textsc{ess} to the complement of \textsc{inc}.
The seventh is the reduction for \textsc{prec}, and it is the most demanding.
An accepted instance of \textsc{prec} requires one positive and one negative answer, so a reduction from a single problem would yield either $\mathsf{NP}$-hardness or $\mathsf{coNP}$-hardness, but not $\mathsf{DP}$-hardness.
We follow the standard route for $\mathsf{DP}$-hardness, the one used for the pairs of propositional formulas whose first component is satisfiable and whose second one is not~\cite{DBLP:books/daglib/0018514}: an arbitrary language of $\mathsf{DP}$ is the intersection of a language of $\mathsf{NP}$ and a language of $\mathsf{coNP}$, and each word is mapped to an instance of \textsc{ess} for the first language and to an instance of the complement of \textsc{ess} for the second one.
The two instances are then merged into a single instance of \textsc{prec}, which must be accepted if, and only if, the first one is accepted and the second one is rejected.
Essentially, the merged selective knowledge base keeps the two datasets disjoint, with $\top$ as the only predicate they share; its unit is the Cartesian product of the two units, which is why this reduction needs an arity of at least two; and the summary of each of its tuples is the union of the summaries of the two halves of the tuple.
Under these conditions every explanation of the merged unit is the conjunction of one explanation per part, on disjoint variables, so that each of the two tests is decided on its own part.
The same construction, applied to an arbitrary language of $\mathsf{DEXP}$, gives the broad case.}

\subsection{Positioning within the literature}\label{sec:positioning}

\textcolor{black}{Two classical facts on conjunctive queries underlie our constructions.
Chandra and Merlin~\cite{chandra1977optimal} characterized the containment of one conjunctive query in another one by the existence of a containment mapping; in the presence of constants, that mapping is the $\C$-homomorphism used throughout this paper~\cite{DBLP:books/aw/AbiteboulHV95}.
The direct product of the datasets in which some given tuples occur, read as a constant-free conjunctive formula, is the most specific conjunctive query that returns those tuples~\cite{DBLP:conf/cp/Willard10,DBLP:conf/pods/CateDFL23}, and Algorithm~\ref{alg:connectedCan} starts from that product.}

\textcolor{black}{Among the settings recalled in Section~\ref{sec:intro}, two are the closest to ours.
Petrova et al.~\cite{DBLP:conf/semweb/PetrovaSGH17} compute the direct product of a knowledge graph with itself, keep the connected component that contains the pair of the two entities, and replace every pair of identical constants with that constant; they prove that the result is a most specific similarity query, unique up to equivalence.
Since their queries may have a variable in the position of a property, a knowledge graph is read here as a dataset over a single ternary predicate; with an empty ontology and a summary selector that returns the whole dataset, two entities then form a unit of two unary tuples, and that query is equivalent to its canonical characterization.
The construction of Amendola et al.~\cite{AMENDOLA2024120331} encoded by Algorithm~\ref{alg:connectedCan} performs the same steps for units of any arity and any number of tuples, on the summaries of the tuples rather than on the whole dataset.
The one delicate case, in which a term of the product stands for a free variable and is also a constant of the dataset, arises for two entities only when they coincide: the algorithm of 2017 does not cover it, and the extension of Petrova et al.~\cite{DBLP:conf/semweb/PetrovaKGH19} adds the patterns in which some occurrences of the entity are replaced with the answer variable, which is what the cloned atoms of Algorithm~\ref{alg:connectedCan} do in general.
Colucci et al.~\cite{DBLP:journals/ws/ColucciDGS16} compute the least common subsumer of two rooted RDF-graphs in time quadratic in the size of the two descriptions, and they make the choice of the triples that describe each resource an explicit parameter; that parameter is one concrete summary selector in the sense of Section~\ref{sec:preliminaries}.
Neither work considers essential expansions or the navigation tasks built on them.}

\textcolor{black}{The requirements stated in Section~\ref{sec:strategies} have no counterpart in those works, nor in the results on products.
Those works take the description of each object as part of the input, while each of our reductions must build a summary selector.
The results on products concern constant-free queries, and the canonical query of a direct product need not be nearly connected, while our characterizations may contain constants and must be nearly connected; the new constants and the self-loop instance of the reduction for \textsc{ess} in the broad case, described in Section~\ref{sec:techniques}, serve exactly these two purposes.
Petrova et al. and Colucci et al. compare two entities, that is, a unit of two unary tuples, while our tasks take units of any number of tuples and of any arity.
Finally, the class $\mathsf{DP}$ arises here for a different reason than in the study of the conjunctive queries that return all the tuples of one given set and none of the tuples of another one~\cite{DBLP:conf/pods/CateDFL23}: there the input carries examples of both kinds, while here, inside a single knowledge base, one characterization must map into the other one and the other one must not map back.}

\textcolor{black}{Where the settings overlap, the known bounds are consistent with ours.
The polynomial algorithms of Colucci et al. and of Petrova et al. are consistent with the membership of \textsc{can} in $F\mathsf{P}$ in the medium case, which covers units of two tuples.
Petrova et al. prove $\mathsf{coNP}$-complete the existence of a query that holds for one entity and not for another one~\cite{DBLP:conf/semweb/PetrovaSGH17}, and the existence of a query whose answers are exactly two given entities~\cite{DBLP:conf/semweb/PetrovaKGH19}; both problems amount to negative tests of the kind that \textsc{inc} performs, and \textsc{inc} is $\mathsf{coNP}$-complete in the medium case.
The agreement concerns the classes and not the proofs.
A difference between two explanation languages carries no prediction about the cost of the tasks: the complexity could increase, it could remain the same, and it could also decrease, as it would if disjunction were allowed, since the most specific explanation of a set of tuples would then be the disjunction of their summaries, read as formulas, and no search would be needed.
Which of the three cases occurs for the language adopted here is the outcome of the analysis of Sections~\ref{sec:computing} and~\ref{sec:tasks}.}

\section{Computing Characterizations and Essential Expansions}\label{sec:computing}

\textcolor{black}{This section computes the two characterizations of a unit and decides membership in its essential expansion. We first fix the notation that every construction of this section, and of the next one, relies upon.}

Let us establish some fundamental notation. Consider the partition of set $\C$ into $\{\C_f,\C_\ell\}$, where $\C_\mathit{f}$ represents flat constants (e.g., $\mathsf{a}$, $\mathsf{a}_1$, $\mathsf{a}_2$) and $\C_{\ell}$ $=$ $\{c^s\!:c\!\in\!\C_f \wedge s>1\}$ represents lifted constants.
We further categorize $\C_\mathit{f}$ into \emph{reserved} and \emph{input} constants. The \emph{reserved} constants are denoted as $\C_\mathit{re}$ $=$ $\{\mathsf{alias}\}$ $~\cup~$ $\{\mathsf{a}_i, \mathsf{b}_i\!:\!i\!>\!0\}$, and the remaining flat constants form the \emph{input} set $\C_\mathit{in}$ $=$ $\C_f\!\setminus\!\C_\mathit{re}$. Similarly, predicates are divided into \emph{reserved} ($\PS_\mathit{re}$) and \emph{input} ($\PS_\mathit{in}$) sets, ensuring specific symbols are exclusively used in our constructions.

\subsection{Canonical Characterizations (\textsc{can} problem)}

This section is entirely devoted to the computation of the canonical characterization associated with an initial $n$-ary unit $\unit = \{\tau_1, \ldots, \tau_m\}$.
To this end, we first recall some key notions, and then present Algorithm~\ref{alg:connectedCan}, which encodes the construction of $\mathit{can}(\unit, \kb)$ as defined in Section~3.3 of~\cite{AMENDOLA2024120331}.
Since the algorithm directly reflects that construction, its correctness follows immediately.
Our primary objective here is therefore to analyze its computational cost.

Consider a sequence of $n$-ary tuples $\bar{\tau}_1, \ldots, \bar{\tau}_\ell$.
Their \emph{direct product}, denoted by $\bar{\tau}_1 \otimes \ldots \otimes \bar{\tau}_\ell$, is defined as the sequence of constants $d_{\bar{\bf s}_1}, \ldots, d_{\bar{\bf s}_n}$, where each $\bar{\bf s}_i$ is the sequence $\bar{\tau}_1[i], \ldots, \bar{\tau}_\ell[i]$.
Given $k$ datasets $D_1, \ldots, D_k$, their \emph{direct product} is the dataset:
\begin{equation*}
\begin{split}
    D_1 \otimes \ldots \otimes D_k = \Big\{ & p(\langle c_1^1, \ldots, c_1^n\rangle \otimes \ldots \otimes \langle c_k^1, \ldots, c_k^n\rangle) : \\
    & p(c_1^1, \ldots, c_1^n) \in D_1, \ldots, p(c_k^1, \ldots, c_k^n) \in D_k \Big\}.
\end{split}
\end{equation*}
    \begin{algorithm}[t!]
	\DontPrintSemicolon
	\KwInput{$\kb  = (K, \varsigma)$ and $\unit = \{\tau_1,...,\tau_m\}$.}
	$d_{{\bf s}_1},...,d_{{\bf s}_n} := \,
 \tau_1 \otimes ... \otimes \tau_m$

	{\bf print} $x_{{\bf s}_1},...,x_{{\bf s}_n} \leftarrow  \top(x_{{\bf s}_1}) \wedge ... \wedge \top(x_{{\bf s}_n})$\;
	\For{$\alpha \in \varsigma(\tau_1) \otimes ... \otimes \varsigma(\tau_m)$}
	{
		\For{$\beta \in \{\alpha\} ~\cup~ \mathit{clones}(\alpha)$}
		{
			\If{$\mathsf{NearCon}(\beta,\unit,\kb) == \mathsf{Yes}$}
			{
				{\bf print} $\wedge$  $\mu(\beta)$\;
			}
		}
	}
	\caption{$\mathsf{BuildCan(\unit,\kb)}$ }\label{alg:connectedCan}
\end{algorithm}

\begin{algorithm}[t!]
	\DontPrintSemicolon
	\KwInput{$\kb  = (K, \varsigma)$, $\unit = \{\tau_1,...,\tau_m\}$ and an atom $\beta$.}
	{\bf guess}  $d_{\bf s} \in \tau_1 \otimes ... \otimes \tau_m$ \;
	\If{$x_{\bf s} \in \Dom_{\{\mu(\beta)\}}$ }
	{ $\mathsf{accept~this~branch}$ }
	{\bf guess} $\alpha \in \varsigma(\tau_1) \otimes ... \otimes \varsigma(\tau_m)$\;
	\If{$\Dom_{\{\mu(\alpha)\}} \cap \Dom_{\{\mu(\beta)\}} \neq \emptyset$}
	{
	    \If{$\mathsf{NearCon}(\alpha,\kb,\unit) == \mathsf{Yes}$}
	    {	   $\mathsf{accept~this~branch}$
	    }
	}
	$\mathsf{reject~this~branch}$
	\caption{$\mathsf{NearCon}(\beta,\unit,\kb)$}\label{alg:connectedness}
\end{algorithm}

Line 1 of Algorithm~\ref{alg:connectedCan} constructs the sequence $d_{\mathbf{s}_1}, \ldots, d_{\mathbf{s}_n}$ of constants, resulting from the direct product of the individual summaries.  Their set is denoted by $\FreeConst$. These constants are placeholders for the free variables of the canonical characterization.
Line 2 prints the free variables of $\mathit{can}(\unit, \kb)$ along with $n$ atoms that always appear in it.
Line 3 builds the set $P$, consisting of the direct product of the summaries.
Line 4 computes the set $C$ of the \emph{cloned atoms} of Amendola et al.~\cite{AMENDOLA2024120331}, which contain terms that play a dual role: they belong both to the constants treated as free variables and to actual constants in the dataset. In such cases, multiple copies of the same atom are needed; some where terms act as constants, others where they represent free variables. Each $\alpha = p(t_1, \ldots, t_k) \in P$ gives rise to a set $\mathit{clones}(\alpha)$ containing at most $2^{|p|}$ atoms.
Line 5 checks, for each $\beta \in P \cup C$, whether $\mu(\beta)$ is connected (directly or indirectly) to some atom containing a free variable. The function $\mu$ maps each constant of the form $d_{\mathbf{s}}$ to a variable $x_{\mathbf{s}}$, for every sequence $\mathbf{s}$ of constants, and acts as the identity otherwise. If the connection holds, Line 6 appends $\mu(\beta)$ to the output.
In particular, $\mathsf{NearCon}$ is implemented by Algorithm~\ref{alg:connectedness}, which mimics nondeterministic graph reachability: given an atom $\beta$, it checks whether $\mu(\beta)$ contains a free variable; if so, it accepts. Otherwise, it guesses an \emph{adjacent} atom $\alpha$ and recursively repeats the check.

\textcolor{black}{The dataset $P$ built at Line 3 is the direct product of the summaries of the tuples of $\unit$, and it is worth making explicit what separates it from $\mathit{can}(\unit,\kb)$. Read as a constant-free formula, the direct product of the datasets in which some given tuples occur is the most specific conjunctive query returning those tuples~\cite{DBLP:conf/cp/Willard10,DBLP:conf/pods/CateDFL23}. Here it cannot play that role, and Lines 4 and 5 are precisely what makes the difference: the former keeps, in both readings, the atoms in which a term also occurs as a constant of the dataset, while the latter discards the atoms that are not connected to any free variable. As a consequence, the nearly connected part of the direct product, read as a constant-free formula, can be mapped into $\mathit{can}(\unit,\kb)$, whereas the converse does not hold in general, since a $\C$-homomorphism cannot send a constant to a variable. The working example of Section~\ref{sec:intro} shows this: the characterizations in Figure~\ref{fig:KG3} mention the entity $\mathsf{US}$, which the constant-free reading of the direct product would replace by a variable.}

For notational convenience, we denote by $[j]$ the set $\{1, \ldots, j\}$ for every $j \in \mathbb{N}^+$.
The following result strengthens Theorems~4 and 5 from Section~5 of~\cite{AMENDOLA2024120331}.

\begin{theorem}\label{thm:CAN-RES}
$\textsc{can}$ belongs to $F\mathsf{PSPACE} \setminus F\mathsf{P}^{\mathsf{PH}}$ in the broad case and to
$F\mathsf{P}$ both in the medium and handy case.%
\end{theorem}

\begin{proof}
We begin by considering the upper bounds in the broad case of the problem.
In this scenario, Algorithm~\ref{alg:connectedCan} runs in $F\mathsf{PSPACE}$ while utilizing $\mathsf{NearCon}$ as an oracle within $\mathsf{NPSPACE}$. We achieve this by sequentially considering each constituent atom within the exponentially large structure associated with the direct product, rather than materializing the entire structure at once.
Since we know that $F\mathsf{PSPACE}^{\mathsf{NPSPACE}} = F\mathsf{PSPACE}$, the result follows directly.
Turning to the lower bounds, to show the unavoidability of the exponentiality of $\mathit{can}(\unit,\kb)$ in the broad case, we establish the following result which is of its own importance:
Let $c$ be the constant $2^\omega$. It holds that
$
|\mathit{can}(\unit,\kb)| \leq c \cdot \prod_{i\in [m]} |\varsigma(\tau_i)|.
$
In particular, there exists a family $\{(\unit_{\bar{m}},\kb_{\bar{m}})\}_{{\bar{m}}>1}$, where $\unit_{\bar{m}} = \{\bar{\tau}_1,...,\bar{\tau}_{\bar{m}}\}$ is a unary unit, and $\kb_{\bar{m}} = (K_{\bar{m}}, \bar{\varsigma})$ is an SKB with $K_{\bar{m}}=(D_{\bar{m}}, O_{\bar{m}})$, $|\bar{\varsigma}(K_{\bar{m}},\bar{\tau}_j)| > 2$, $j\in\{1,...,m\}$, such that
$|\mathit{can}(\unit_{\bar{m}},\kb_{\bar{m}})| = 2^{1-\bar{m}} \cdot \prod_{i\in [{\bar{m}}]} |\bar{\varsigma}(\bar{\tau}_i)|.$
The above result clearly implies that such an object cannot be constructed in $F\mathsf{P}^{\mathsf{PH}}$.
We now simply have to explicitly construct the family $\{(\unit_{\bar{m}},\kb_{\bar{m}})\}_{{\bar{m}}>0}$ and prove that $|\mathit{can}(\unit_{\bar{m}},\kb_{\bar{m}})| = 2^{1-\bar{m}} \cdot \prod_{i\in [{\bar{m}}]} |\bar{\varsigma}(\bar{\tau}_i)|$.
To achieve this, let us proceed step by step:

\textbf{(i)} \emph{Construction of the Family $\{(\unit_{\bar{m}},\kb_{\bar{m}})\}_{{\bar{m}}>0}$.}
Let $\pr_{\bar{m}}$ be the $\bar{m}$-th prime number and for each $i\in\mathbb{N}$ let $\Gamma_{i} = \{\mathsf{r}(\mathsf{c}_1^{i},\mathsf{c}_2^{i}),...,\mathsf{r}(\mathsf{c}_{\pr_{i}-1}^{i},\mathsf{c}_{\pr_{i}}^{i}),\mathsf{r}(\mathsf{c}_{\pr_i}^{i},\mathsf{c}_{1}^{i})\} ~\cup~ \{\top(\mathsf{c}_j^{i}) : j\in [\pr_i]\}.$
Define $D_{\bar{m}} = ~\cup~_{i \in [\bar{m}]} \Gamma_{i}$, $O_{\bar{m}}=\emptyset$, $K_{\bar{m}}=(D_{\bar{m}},O_{\bar{m}})$ and $\bar{\varsigma}$ such that $\bar{\varsigma}(K_{\bar{m}},\bar{\tau}_i) = \Gamma_{i}$ with $\bar{\tau}_i = \langle \mathsf{c}_1^i\rangle$ for $i\in[\bar{m}]$; finally, consider $\unit_{\bar{m}} = ~\cup~_{i \in [\bar{m}]} \{\langle \mathsf{c}_1^i\rangle\}$.

\textbf{(ii)} \emph{Verification of $ |\mathit{can}(\unit_{\bar{m}},\kb_{\bar{m}})| = 2^{1-\bar{m}} \cdot \prod_{i\in [{\bar{m}}]} |\bar{\varsigma}(\bar{\tau}_i)|$.}
For the construction of $\mathit{can}(\unit_{\bar{m}},\kb_{\bar{m}})$, we define the dataset $P = \bar{\varsigma}(\bar{\tau}_1) \otimes ... \otimes \bar{\varsigma}(\bar{\tau}_{\bar{m}}) = \bar{\varsigma}(\langle c_1^1\rangle) \otimes ... \otimes \bar{\varsigma}(\langle c_1^{\bar{m}}\rangle)$.
Let $N=\prod_{i\in[\bar{m}]}[\pr_{i}]$.
For every $n\in[N]$, let $\mathsf{rem}(n,i)$ be the remainder of the division $n/\pr_i$. We define $f_{\bar{m}}(n)=c_{s(1,n)}^{1},...,c_{s(\bar{m},n)}^{\bar{m}}$, where
$$
s(i,n)=\begin{cases}
   \mathsf{rem}(n,i) & \text{if } \mathsf{rem}(n,i)\neq0  \\
   \pr_i & \text{otherwise}
\end{cases}
$$

Given the above definitions, we get that:
\begin{equation*}
    \begin{split}
        P = \ & \{\top(d_{f_{\bar{m}}(n)}) \mid n\in[N]\} \\
            & \cup \{r(d_{f_{\bar{m}}(N)},d_{f_{\bar{m}}(1)})\} \\
            & \cup \{r(d_{f_{\bar{m}}(n)},d_{f_{\bar{m}}(n+1)}) \mid n\in[N-1]\}.
    \end{split}
\end{equation*}
We now construct $\mu$ as the mapping $\{ d_{\bf s} \mapsto g(d_{\bf s}) : d_{\bf s} \in \Dom_P\}$, where $g(d_{\bf s})$ is the mapping $\{ d_{\bf s} \mapsto x_{\bf s} : d_{\bf s} \in \FreeConst\} ~\cup~ \{ d_{\bf s} \mapsto y_{\bf s} : d_{\bf s} \not\in  \FreeConst\}$, since here no term of the direct product also occurs as a constant of the dataset, and hence no cloned atom arises.
Finally, we can define:
\begin{equation*}
    \begin{split}
        \Phi(\unit_{\bar{m}},\kb_{\bar{m}}) &= x_{{\bf s}_1} \leftarrow \bigwedge_{p(t_1,\dots,t_k) \in P} p(\mu(t_1),\dots,\mu(t_k)) \\
        &= x_{{\bf s}_1} \leftarrow \bigwedge_{r(t_1,t_2) \in P} r(\mu(t_1),\mu(t_2)) \wedge \bigwedge_{t\in \Dom_P} \top(\mu(t)).
    \end{split}
\end{equation*}
Since the formula $\Phi(\unit_{\bar{m}},\kb_{\bar{m}})$ just constructed is connected, it coincides with $\mathit{can}(\unit_{\bar{m}},\kb_{\bar{m}})$.
Moreover, since $|P|=|\Gamma_{N}|$, $|P|$ and hence the size of $\mathit{can}(\unit_{\bar{m}},\kb_{\bar{m}})$, is exponential concerning the parameter $\bar{m}$.

Moving on to the remaining cases, due to the boundedness of $m$, the following statements hold:
$(i)$ Algorithm~\ref{alg:connectedCan} runs in $F\mathsf{P}$, and $(ii)$  $\mathsf{NearCon}$ runs in $\mathsf{NL}$.
The result then follows since we know that $F\mathsf{P}^{\mathsf{NL}} = F\mathsf{P}$.
This concludes the proof, establishing both upper and lower bounds for the given problem in the various scenarios.
\end{proof}

Following the previous proof we can derive a corollary, a version of which was already present in~\cite{AMENDOLA2024120331}, that will be useful in the following.
\begin{cor}\label{cor:exponential}
    There exists a family $\{(\unit_{\bar{m}},\kb_{\bar{m}})\}_{{\bar{m}}>0}$, where $\unit_{\bar{m}} = \{\bar{\tau}_1,...,\bar{\tau}_{\bar{m}}\}$ is a unary unit and  $\kb_{\bar{m}} = (K_{\bar{m}}, \bar{\varsigma})$ is an SKB with $\bar{\varsigma}(K_{\bar{m}},\bar{\tau}_m) \subset \bar{\varsigma}(K_{\bar{m}+1},\bar{\tau}_{{m}+1})$, such that
\[
|\mathit{core}(\unit_{\bar{m}},\kb_{\bar{m}})| =  |\mathit{can}(\unit_{\bar{m}},\kb_{\bar{m}})| = 2^{1-\bar{m}} \, \cdot \!\! \prod_{i\in [{\bar{m}}]}  |\bar{\varsigma}(\bar{\tau}_i)|.
\]

\end{cor}
\begin{proof}
    The proof of this corollary follows from the proof of Theorem~\ref{thm:CAN-RES} noting that the constructed $$\mathit{can}(\unit_{\bar{m}},\kb_{\bar{m}})=x_{{\bf s}_1} \leftarrow \bigwedge_{r(t_1,t_2) \in P} \ r(\mu(t_1),\mu(t_2))~\wedge~ \bigwedge_{t\in \Dom_P} \ \top(\mu(t))$$ is always a core of itself.
    This last statement follows since a core of a single (finite) cycle is always isomorphic to the cycle itself and the part $~\bigwedge_{r(t_1,t_2) \in P} \ r(\mu(t_1),\mu(t_2))$ forms a cycle of cardinality $N=\prod_{i\in[\bar{m}]}[\pr_{i}]$, where $\pr_{i}$ we are representing the $i$-th prime number as before.
    On the other hand since we are not able to delete any of the atoms in this cycle we are neither able to delete any of the atoms in $~\bigwedge_{t\in \Dom_P} \ \top(\mu(t))$, otherwise we would get an absurd since we would have been able to delete also an element coming from the cycle.
\end{proof}

\textcolor{black}{Corollary~\ref{cor:exponential} says more than the size of one particular characterization. Its input is a unary unit of $\bar{m}$ tuples, one per summary, and the summaries are cycles of pairwise distinct prime lengths. Since the canonical characterization of each member of the family is already a core, and since equivalent formulas of minimum size are unique up to isomorphism, no characterization of those units is smaller. The exponential size is therefore a property of the units themselves, and not of the way in which we compute their characterizations.}
\textcolor{black}{However, as the results recalled from Section~3.3 of~\cite{AMENDOLA2024120331} show, canonical characterizations, although well behaved from the computational point of view, are often hard to read. The next subsection therefore turns to core characterizations.}

\subsection{Core Characterizations (\textsc{core} problem)}
\textcolor{black}{In this section, we study the computational complexity of building a core characterization from a given unit.
Regarding the upper bound, while the minimization of a conjunctive formula to its core relies on the fundamental theory of query folding by Chandra and Merlin~\cite{chandra1977optimal}, we include Algorithm~\ref{alg:buildCore} to make the presentation self-contained. Specifically, we provide a unified procedure that first constructs $\mathit{can}(\mathbf{U},\mathcal{S})$ and subsequently prunes redundant atoms (where, at each iteration, $\varphi'$ represents a folding of $\varphi$) to compute $\mathit{core}(\mathbf{U},\mathcal{S})$.}
\textcolor{black}{In particular}: line 1 constructs $\varphi = \mathit{can}(\unit,\kb)$ and collects its atoms in $A$; line 2 enumerates each $\alpha \in A$; line 3 builds $\varphi'$ from $\varphi$ by removing $\alpha$; line 4 checks whether $\varphi \longrightarrow \varphi'$; if so, line 5 copies $\varphi'$ in $\varphi$; finally, after removing all redundant atoms, line 6 prints $\varphi$.
Similarly to what was said for Theorem~\ref{thm:CAN-RES}, the following result reinforces the statement presented in~\cite{AMENDOLA2024120331} in Section~5 Theorems ~4 and 5.

Before formalizing the complexity of \textsc{core}, we need some preliminary notions.
A \emph{graph} is an ordered pair $G = (N,A)$, where $N$ is the set of {\em nodes}, while $A \subseteq N\times N$ is the set of {\em arcs}.
Two distinct nodes $u$ and $v$ are said \emph{adjacent} if $(u,v) \in A$ or $(v,u) \in A$.
If $A = N \times N \setminus \{(u,u):u \in N\}$, then $G$ is said to be {\em complete}.
As common, $G$ can be transparently viewed as the structure $\{\mathsf{edge}(u,v):(u,v) \in A\}$.
Therefore, the notions of homomorphisms, isomorphisms, or cores directly apply to graphs.
If $A$ is symmetric, namely $(u,v) \in A$ implies $(v,u) \in A$, then $G$ is said to be {\em undirected}.
For  {\em undirected graphs}, it is also common the notation $G = (V,E)$, where $V$ are the {\em vertices} while $E$ is the set of {\em edges} each of the form $\{u,v\}$, which intuitively corresponds to having both the arcs $(u,v)$ and $(v,u)$.

\begin{algorithm}[t!]
	\DontPrintSemicolon
	\KwInput{$\kb  = (K, \varsigma)$ and $\unit = \{\tau_1,...,\tau_m\}$.}
		$\varphi := \mathsf{BuildCan}(\unit,\kb)$ \ \ \ \mbox{and} \ \ \ $A := \mathit{atm}(\varphi)$\;
	    \For{$\alpha \in A$}
	    {
	        $\varphi' := \mathit{remove}(\alpha,\varphi)$\;
	        \If{$\varphi \longrightarrow \varphi'$}
	        {$\varphi := \varphi'$}
	    }
	    {\bf print} $\varphi$
	\caption{$\mathsf{BuildCore(\unit,\kb)}$ }\label{alg:buildCore}
\end{algorithm}

\begin{theorem}\label{thm:corecore}
$\textsc{core}$ is in $F\mathsf{EXP}^{\mathsf{NP}}\setminus F\mathsf{P}^{\mathsf{PH}}$ in the broad case, in $F\mathsf{P}^{\mathsf{NP}}$ in the medium case and in $F\mathsf{P}$ in the handy case.
Unless $\mathsf{NP}=\mathsf{coNP}$, $\textsc{core} \not\in F\mathsf{P}$ in the medium case.
\end{theorem}
\begin{proof}
\textcolor{black}{Regarding the upper bounds, we rely on Algorithm~\ref{alg:buildCore}.
As already observed, its correctness follows directly from the classical theory of query folding~\cite{chandra1977optimal}; hence, it remains only to analyze its asymptotic complexity.
In the broad case, the algorithm runs in $F\mathsf{EXP}$ using an oracle in $\mathsf{NP}$ to check whether $\varphi \longrightarrow \varphi'$ holds.
In the medium case, Algorithm~\ref{alg:buildCore} runs in $F\mathsf{P}$, again relying on an $\mathsf{NP}$ oracle to perform the containment check $\varphi \longrightarrow \varphi'$.
Finally, in the handy case, the computation reduces to polynomial time ($F\mathsf{P}$), as both the size of $\varphi$ and the number of its variables are bounded.}

About the lower bounds, we need to consider both the broad and the medium case.
For what concerns the broad case, Corollary~\ref{cor:exponential} leads us to conclude that in general such an object, being exponential, cannot be constructed in $F\mathsf{P}^{\mathsf{PH}}$.
For what concerns the medium case instead, let $\textsc{core-identification}$ be the $\mathsf{DP}$-complete problem reported in~\cite{DBLP:journals/tods/FaginKP05}:
{\em given a pair $(S',S)$ of finite structures such that $S' \subseteq S$, is $S'$ isomorphic to $\mathit{core}(S)$}?%
\footnote{A {core} of a structure $S$ is any $\bar{S} \subseteq S$  such that $S \longrightarrow \bar{S}$, but there is no $\bar{S}' \subset \bar{S}$ such that $\bar{S} \longrightarrow \bar{S}'$.
Since $\bar{S}$ is unique up to isomorphism~\cite{DBLP:journals/tods/FaginKP05}, we simply refer to {\em the} core of $S$, denoted by $\mathit{core}(S)$.}

Let $\mathbb{NT}$ be all constant-free formulas in $\WCQ$ closed under $\top$, and $\textsc{nt-core}$ be the analogous of $\textsc{core-identification}$ for $\mathbb{NT}$ formulas, in other terms, given a pair $(\varphi',\varphi)$ of $\mathbb{NT}$ formulas such that $\mathit{atm}(\varphi') \subseteq \mathit{atm}(\varphi)$, is $\varphi'$ isomorphic to $\mathit{core}(\varphi)$?
Indeed, $\textsc{nt-core}$ remains $\mathsf{DP}$-hard. To show this last statement, first notice that exploiting the proof of Theorem 4.2 in~\cite{DBLP:journals/tods/FaginKP05}, we can consider, as input for $\textsc{core-identification}$, two undirected graphs.

Now consider the following composition of mappings $(G',G)\mapsto(\varphi',\varphi)\mapsto(\varphi''',\varphi'')$, where $G,G'$ are two undirected graphs, $\varphi,\varphi',\varphi''$ and $\varphi'''$ are all constants-free \textcolor{black}{conjunctive} formulas. Moreover, $G'\subseteq G$, $\mathit{atm}(\varphi') \subseteq \mathit{atm}(\varphi)$, $\mathit{atm}(\varphi''') \subseteq \mathit{atm}(\varphi'')$, $\varphi$ and $\varphi'$ are over a single binary and symmetric relation, call it $\mathsf{edge}$, and both $\varphi'''$ and $\varphi''$ are $\mathbb{NT}$ formulas.
First,  $\varphi$ (resp., $\varphi'$) is the conjunctive formula without free variables containing two atoms, namely $\mathsf{edge}(x_u,x_v)$ and $\mathsf{edge}(x_v,x_u)$, for each  $\{u,v\}$ in $G$ (resp., $G'$).
Second, we define:
\begin{align*}
    \varphi'' &= x \leftarrow \bigwedge_{\alpha \in \mathit{atm}(\varphi)} \alpha \wedge \bigwedge_{y\in\Dom(\varphi)}s(x,y) \wedge \bigwedge_{z\in\Dom(\varphi) \cup \{x\}}\top(z) \\
    \text{and} \quad \varphi''' &= x \leftarrow \bigwedge_{\alpha \in \mathit{atm}(\varphi')} \alpha \wedge \bigwedge_{y\in\Dom(\varphi')}s(x,y) \wedge \bigwedge_{z\in\Dom(\varphi') \cup \{x\}}\top(z)
\end{align*}

\noindent where $s$ is some fresh predicate different from $\mathsf{edge}$.
We are going to show that $G'$ is the core of $G$ if, and only if, $\varphi'$ is the core of $\varphi$ if, and only if, $\varphi'''$ is the core of $\varphi''$.
The fact that $\varphi'$ is the core of $\varphi$ if $G'$ is the core of $G$, is obvious.
For the other direction, say that $\varphi'$ is the core of $\varphi$ and say by contradiction that $\varphi'''$ is not the core of $\varphi''$, since by construction $\varphi'''\mapsto\varphi''$ the only way possible is that it exists $\tilde{\varphi}\in \mathbb{NT}$ such that $\mathit{atm}(\tilde{\varphi})\subset\mathit{atm}(\varphi'')$ and $\tilde{\varphi}$ is the core of $\varphi'''$ (that coincides by definition with the core of $\varphi''$).
Let $\alpha$ be an atom in $\mathit{atm}(\varphi'')\setminus\mathit{atm}(\tilde{\varphi})$, such an atom has to exist in order to have $\mathit{atm}(\tilde{\varphi})\subset\mathit{atm}(\varphi'')$. By construction, the only 3 possibilities for this atom are $(i)$ $\alpha=\mathsf{edge}(z,z')$ for some $z,z'\in\Dom_{\mathit{atm}(\varphi'')}$, $(ii)$ $\alpha=s(x,z)$ for some $z\in\Dom_{\mathit{atm}(\varphi'')}$ and $(iii)$ $\alpha=\top(z)$ for some $z\in\Dom_{\mathit{atm}(\varphi'')}$.

Now we will prove that for any of these 3 possibilities, we will end up with a contradiction. First, say that $\alpha$ is of the form $\mathsf{edge}(z,z')$ for some $z,z'\in\Dom_{\mathit{atm}(\varphi'')}$, then by definition it exists a homomorphism $h$ from $\mathit{atm}(\varphi'')$ to $\mathit{atm}(\varphi'')$ such that at least one of the following two does hold: either $h(z)\neq z$ or $h(z')\neq z'$. Say that $h(z)\neq z$, the other implication can be solved by using a similar argument. Then we could use the same homomorphism $h$ on $h$ on the variables present in $\varphi'$ and thus obtain that $\tilde{\varphi'}$ exists such that $\mathit{atm}(\tilde{\varphi'})\subset\mathit{atm}(\varphi')$, contradicting the assumption that $\varphi'$ is a core. All other remaining possibilities give rise to the same contradiction following the same construction.

On the other hand, say that $\varphi'$ is not the core of $\varphi$, then obviously $\varphi'''$ is not the core of $\varphi''$. Now we show that for each $\varphi \in \mathbb{NT}$, it is possible to construct in $F\mathsf{P}$ a pair $(\kb, \unit)$ such that both $\varphi \simeq \mathit{can}(\unit,\kb)$ and $\mathit{core}(\varphi) \simeq \mathit{core}(\unit,\kb)$ hold. In order to do so, starting from a $k$-ary $\varphi \in \mathbb{NT}$, without loss of generality consider its free variables to be $x_1,...,x_k$, and consider the following knowledge base $K=(\db,\emptyset)$, where the dataset $\db=\{p(cz_1,...,cz_n)~:~p(z_1,...,z_n)\in\mathit{atm}(\varphi)\}~\cup~\{p(\mathsf{alias},...,\mathsf{alias})~:~p(z_1,...,z_n)\in\mathit{atm}(\varphi)\}$, $\unit=\{\langle cx_1,...,cx_k\rangle,\langle\mathsf{alias},...,\mathsf{alias}\rangle\}$. As the last thing, we have to define our summary selector $\varsigma$ as the following function $\varsigma(K,\tilde{\tau})$ that returns $\db$ whatever the input tuple $\tilde{\tau}$ will be. The really important thing to note in this construction is that the initial database domain is bipartite. In particular, there exist no $p(z_1,...,z_n)\in\db$ such that $z_i=\mathsf{alias}$ for some $i\in[n]$ and then there exist $j\in[n]$ such that $z_i\neq z_j$. This important property tells us a priori an important thing about the construction of the canonical explanation of the unit we built, in particular no term of a connected part of $P$ that contains $d_s$ with $s=cx_i,\mathsf{alias}$ and $i\in[k]$ also occurs as a constant of the dataset, so that no cloned atom arises there, and this tells us we will have no constants at all in the set of atoms of $\mathit{can}(\unit,\kb)$; moreover, $\mathit{atm}(\mathit{can}(\unit,\kb))$ will be isomorphic to $\mathit{atm}(\varphi)$ by construction. To note this, the simplest way is to see that each constant of the direct product that is constructed using the algorithm provided in Section 3.2 of~\cite{AMENDOLA2024120331} will have the form $d_{c,\mathsf{alias}}$ for some constant $c$ present in the union of the connected parts of the dataset containing the constants $cx_1,...,cx_k$ but by construction this is isomorphic to the starting \textcolor{black}{conjunctive} formula.

We can now reduce $\textsc{nt-core}$ to $\textsc{core}$ in the medium case. From a pair $(\varphi',\varphi)$ of formulas in $\mathbb{NT}$, do: $(i)$ construct in $F\mathsf{P}$ a pair $(\kb,\unit)$ such that $\varphi \simeq \mathit{can}(\unit,\kb)$; $(ii)$ construct $\mathit{core}(\unit,\kb)$; and $(iii)$ check in $\mathsf{NP}$ whether $\varphi' \simeq \mathit{core}(\unit,\kb)$. If $\textsc{core}$ were in $F\mathsf{P}$ in the medium case, then step $(ii)$ would also be in $F\mathsf{P}$ and, hence, $\textsc{nt-core}$ would be in $\mathsf{NP}$, which is impossible unless $\mathsf{NP} = \mathsf{coNP}$.
\end{proof}

\textcolor{black}{Two steps of the proof above have no counterpart in~\cite{DBLP:journals/tods/FaginKP05}, and it is worth isolating them, starting from the inputs involved. The problem $\textsc{core-identification}$ takes a pair of finite structures, one included in the other; $\textsc{nt-core}$ takes a pair of formulas of $\mathbb{NT}$, one whose atoms are included in those of the other; while $\textsc{core}$ takes a selective knowledge base and a unit, and no formula at all. The first step bridges the first two problems: attaching every variable to a single free variable by means of a fresh predicate, and closing the result under $\top$, does not change the core, which is what makes $\mathsf{DP}$-hardness survive the change of language. The second step bridges the last two, and it goes in the direction opposite to the task itself: instead of computing the core characterization of a given unit, it builds, from an arbitrary formula of $\mathbb{NT}$, a selective knowledge base and a unit of two tuples whose canonical characterization is that formula and whose core characterization is its core. Bounding the number of tuples of a unit, as the medium assumption does, therefore does not restrict the shape of its characterizations. Without this step, the hardness of core identification would say nothing about $\textsc{core}$, whose input is a unit and not a formula.}

As we already discussed, building characterizations is not all that we set out to do. We are in fact also interested in understanding when a given tuple shares more, less nexus of similarity with a given unit compared to another tuple or if, perhaps, these nexus are incomparable with each other. The next section will introduce some useful gadgets that will help us in answering these questions.

\subsection{\textcolor{black}{Gadgets for the Lower Bounds}}\label{sec:gadgets}
\textcolor{black}{The reductions of the following sections have to do more than encode a hard problem. They have to return a selective knowledge base, hence a summary selector defined on every tuple of the domain; they have to keep control of which constants may appear in the characterizations they build; and they have to work for units of any fixed arity, not just for unary ones. Rather than dealing with these requirements anew in each proof, we collect here the devices that are used throughout, each of them illustrated by an example. Two of them mark constants: $\mathsf{focus}$ singles out a set of constants, and twins attach to each constant a fixed number of fresh copies, which is what allows a unary unit to be turned into a unit of any arity while keeping its characterizations under control. Three more act on datasets and tuples: one duplicates a chosen constant, one recovers the tuple of the original knowledge base a given tuple comes from, and one combines two units into their Cartesian product. We close with a construction on graphs and with a property of graph colorings, both used in the medium case only, which we state and prove for completeness. \textcolor{black}{Each device serves a specific reduction: the twins are used by every reduction of Sections~\ref{sec:ess} and~\ref{sec:tasks}, the construction on graphs and the property of graph colorings only by \textsc{ess} in the medium case, the duplication of a constant and the recovery of the original tuple by \textsc{sim} and by \textsc{inc}, the focus predicate only by \textsc{inc}, and the Cartesian product of two units only by \textsc{prec}.}}

The concept of \emph{focus} and \emph{twins} is fundamental. For any $k>0$ and any set $C$ of flat constants, let
\[\mathit{fc}(C) = \{\mathsf{focus}(c) : c \in C\}
\qquad \mbox{and} \qquad
\mathit{tw}(C,k) = \{\top(c^s),\mathsf{twin}_s(c,c^s) : c \in C \wedge s \in [2 \ldots k]\}.\] Intuitively, the first set gives greater importance, or \emph{focus}, to a subset of the flat constants, while the second creates $k-1$ twins of each constant in $C$, each of them linked to it by a dedicated predicate.
To better grasp this, consider an example that illustrates the construction and significance of these twins.
    Let $D=\{r(1,2)\}$, clearly $\Dom_D=\{1,2\}$. We want to know what $\mathit{fc}(\Dom_D)$ and $\mathit{tw}(\Dom_D,2)$ look like. By their own definition we have that $\mathit{fc}(\Dom_D)=\{\mathsf{focus}(1),\mathsf{focus}(2)\}$ and $\mathit{tw}(\Dom_D,2)=\{\mathsf{twin}_2(1,1^2)$ $,\mathsf{twin}_2(2,2^2), $ $ \top(1^2), $ $\top(2^2)\}$.
Starting from a unary tuple $\tau=\langle a \rangle$, and an integer $k\geq 1$, let  $\tau^k$ $=$ $\langle a,a^2,...,a^k\rangle$.
From a unary unit $\unit$ $=$ $\{\tau_1,...,\tau_m\}$,
let $\unit^k$ be  $\{\tau_1^k,...,\tau_m^k\}$.
    An example of the above is the following: consider the following unit $\unit=\{\langle1\rangle,\langle2\rangle\}$, then we can construct $\unit^3$ that is equal to $\{\langle1\rangle^3,\langle2\rangle^3\}=\{\langle1,1^2,1^3\rangle,\langle2,2^2,2^3\rangle\}$.
\begin{lemma}\label{lem:twins}
\textcolor{black}{Let $\kb$ be an SKB whose dataset contains $\mathit{tw}(C,k)$, where $C$ is the set of the constants of a unary unit $\unit$, and let $\varphi$ be a $\WCQ$ formula. Then $\varphi$ is a characterization of $\unit$ with respect to $\kb$ if, and only if, the formula
$$x_1,\ldots,x_k\leftarrow \bigwedge_{s\in[2\ldots k]}\top(x_s),\mathsf{twin}_s(x_1,x_s) \wedge\bigwedge_{\alpha \in \mathit{atm}(\varphi)}\alpha$$
is a characterization of $\unit^k$ with respect to $\kb$.}
\end{lemma}
\begin{proof}
\textcolor{black}{From a characterization of $\unit$ one obtains the displayed formula by adding, for each $s\in[2\ldots k]$, the two atoms that relate the $s$-th free variable to the first one. By construction these two atoms belong to the summary of every tuple of $\unit^k$ and to no other summary. Conversely, renaming in a characterization of $\unit^k$ every occurrence of a free variable other than the first one yields a characterization of $\unit$.}
\end{proof}

\textcolor{black}{By Lemma~\ref{lem:twins}, each of the reductions of the following sections can be given for a unary unit and then applied to units of any fixed arity.}

From a dataset $D$, let
$\mathit{double}(D,a)$ be the  dataset defined as $\{p(c_1,...,c_{|p|})\!: p(b_1,...,b_{|p|})\!\in\!\db \wedge \mbox{each}~c_i\in g(b_i)\}$, where $g(a)\!=\!\{a,\mathsf{alias}\}$, and $g(b_i)\!=\!\{b_i\}$ whenever $b_i\!\neq\!a$.
Please note that, as the following example will explain, what we will obtain, in general, will not be a completely detached copy of the initial dataset. Instead, it will be a new structure that remains connected to the original.
    If we consider a dataset $D$ to be equal to $\{\mathsf{r}(1,2)\}$, then we have that $\mathit{double}(D,2)$ is equal to $\{\mathsf{r}(1,2), $ $\mathsf{r}(1,\mathsf{alias})\}$.
From a tuple $\tau$ of constants of any arity,
let $\mathit{off}(\tau,a)$ be the tuple of flat constants obtained from $\tau$ by replacing each $c^s$ with $c$ and $\mathsf{alias}$ with $a$.
Intuitively, this operation might be mistaken as the inverse of the previously defined operation.

Although it will be used in that way later on, it is important to note that in general, this would not be guaranteed if the constant $\mathsf{alias}$ were not part of the reserved constants to which we do not have access as potential inputs. The next example will clarify this assertion.
  Let $D=\{r(1,2), r(3,\mathsf{alias}),r(1^4,2^5)\}$, so that $\Dom_D=\{1,2,3,\mathsf{alias},1^4,2^5\}$. Applying $\mathit{off}(\cdot,4)$ to every unary tuple over $\Dom_D$ results in the set $\{\langle 1 \rangle, \langle 2 \rangle , \langle 3 \rangle, \langle 4 \rangle\}$, because $\mathit{off}$ replaces $1^4$ by $1$, $2^5$ by $2$, and $\mathsf{alias}$ by $4$. In particular, this operation works on tuples and not on datasets.
  Let $\tau=\langle\mathsf{alias},1^3,2,b^2\rangle$, clearly is not a possible input tuples for any of our reductions since it uses $\mathsf{alias}$, $1^3$ and $b^2$. What about $\mathit{off}(\tau,b)$? The final result is $\langle b,1,2,b\rangle$ that is a possible input for our reductions.
From two tuples $\langle { \bf t}_1 \rangle$ and $\langle { \bf t}_2 \rangle$ of any arities, let $\langle { \bf t}_1 \rangle \circ \langle { \bf t}_2 \rangle$ be $\langle { \bf t}_1,{ \bf t}_2 \rangle$.
From two units $\unit_1$ and $\unit_2$, let $\unit_1 \times \unit_2$ be the unit $\{\tau_1 \circ \tau_2: \tau_1\in \unit_1 \wedge \tau_2 \in \unit_2\}$.

The operation defined above is nothing more than the classical \emph{Cartesian product}.
In the remainder of the paper, this operation will be carried out in cases where $\Dom_{\unit_1} \cap \Dom_{\unit_2} = \emptyset$.
    Accordingly let $\unit_1=\{\langle  a,b\rangle,\langle b,c \rangle\}$ and $\unit_2=\{\langle 1\rangle\}$, then we have that $\unit_1 \times \unit_2=\{\tau_1 \circ \tau_2: \tau_1\in \unit_1 \wedge \tau_2 \in \unit_2\}=\{\langle a,b,1 \rangle,\langle b,c,1\rangle\}$.
Given some fixed $k\in \mathbb{N}$, a \emph{$k$-coloring} for $G=(V,E)$ is a mapping $V \rightarrow [k]$ such that no two adjacent vertices are mapped to the same element.
It is well-known that $G$ is $k$-colorable if, and only if, there exists a homomorphism from $G$ to $K_k$, where $K_k$ is the complete undirected graph over the set $\{1,\ldots,k\}$ of vertices.
For each $k\geq 1$, $k$\textsc{-col} is the problem of checking whether an undirected graph $G$ admits a $k$-coloring.
For each $k\geq 3$, problem $k$\textsc{-col} is $\mathsf{NP}$-complete~\cite{Kar72}.
We are now ready to introduce a technical notion that we are going to use to determine the computational complexity of \textsc{ess} from the one of $k$\textsc{-col}.

\begin{definition}\label{def:complement}
    Consider a graph ${G}=(V,E)$ and a
    fresh vertex $u$ not already in $V$.
    We define ${G}^{u}$ as the new graph $(V^{u},E^{u})$, where $V^{u}=V~\cup~\{ u\}$ and $E^{u}=E~\cup~\{(u,v),(v,u)~:~v\in V \}$,  called the \emph{complement to $u$} of ${G}$. \mybox
\end{definition}

To illustrate the above definition, let us consider the following dataset $D=\{\mathsf{r}(1,2),$ $\mathsf{r}(1,3),$ $\mathsf{r}(2,3),$ $\mathsf{r}(3,1),$ $\mathsf{r}(3,2),$ $\mathsf{r}(2,1)\}$ that can be seen as a relational representation of a graph ${G}$. Then ${G}^a$ is represented by $D^a=D~\cup~\{\mathsf{r}(a,1),$ $\mathsf{r}(a,2),$ $\mathsf{r}(a,3),$ $\mathsf{r}(1,a),$ $\mathsf{r}(2,a),$ $\mathsf{r}(3,a)\}$.

 \begin{lemma}\label{lem:color1}
A graph $G=(V,E)$ is $k$-colorable if, and only if, $G^a$ is $(k+1)$-colorable.
	\end{lemma}
	\begin{proof}
		$\Longrightarrow$ Say that ${G}$ is $k$-colorable, this means that it exists $h:V\mapsto [k]$ such that no two adjacent vertices are mapped to the same element. From $h$ we can construct $\bar h:V^{a}\mapsto [k+1]$ in the following fashion:
		$$
		\bar h(v)=
		\begin{cases}
			h(v), \ \ \ \forall v\in V\\
			k+1, \ \ \mbox{otherwise}
		\end{cases}
		$$
		$\Longleftarrow$ Say that ${G}^{a}$ is (k+1)-colorable, this means that it exists $\bar h:V^{a}\mapsto [k+1] $ such that no two adjacent vertices are mapped to the same element. Without loss of generality we may say that $\bar h(a)=k+1$. From $\bar h$ we can construct $ h:V\mapsto [k]$ simply by putting
		$$
		h(v)=
		\bar h(v), \forall v\in V
		$$ In fact, since a is adjacent to any other vertex in $V^{a}$ by construction if $\bar h$ is a proper (k+1)-coloring then no other element of $V^{a}$ is mapped to k+1, but this means that every element of $V$ is mapped to $[k]$ by $\bar h$ and thus $h$ is a proper k-coloring of ${G}$.
	\end{proof}

\begin{remark}\label{rem:selector}
\textcolor{black}{In each of the reductions of the following sections, the summary selector of the constructed SKB is either constant, and returns the whole dataset, or obtained from the selectors received as input by a linear scan of the input followed by finitely many calls to them. In both cases it is defined on every tuple of the domain and computable in polynomial time, so that each reduction preserves the assumption under which it is applied.}
\end{remark}

\subsection{Essential Expansion (\textsc{ess} problem)}\label{sec:ess}
\textcolor{black}{We can now analyze \textsc{ess}, from which Section~\ref{sec:tasks} derives the hardness of the three navigation tasks. The reductions below use the devices introduced in Section~\ref{sec:gadgets}.}

\begin{theorem}\label{thm:essfinal}
\textsc{ess} is $\mathsf{NEXP}$-complete, $\mathsf{NP}$-complete and in $\mathsf{P}$ in the broad, medium and handy case, respectively.
In particular, lower bounds  hold already for $\textsc{ess}_k$ with $k>0$, a selective $\kb$ with an empty ontology and a unit $|\unit|>1$.
\end{theorem}

\begin{proof}
For what concerns the upper bounds, we can make use of Algorithm~\ref{alg:InEss}. Let us consider the three cases separately. In the broad case, line 1 runs in exponential time, lines 2 run in nondeterministic exponential time, and lines 3-4 run in exponential time. In the medium case, line 1 runs in polynomial time, lines 2 run in nondeterministic polynomial time, and lines 3-4 run in polynomial time. For the handy case, instead of guessing a mapping as we do in Algorithm~\ref{alg:InEss}, we can simply enumerate all of them. This can be done since, for any $\tau$, the cardinality of all possible $\C$-homomorphisms from $\mathit{atm}(\mathit{can}(\unit,\kb))$ to $\varsigma(\tau)$ is at most polynomial with respect to the size of $\Dom_{\mathit{ent}(K)}$~\cite{AMENDOLA2024120331}.
Hence, since all possible $\C$-homomorphisms are polynomially many, the resulting algorithm will run in polynomial time.

For what concerns the lower bounds, we will give two distinct constructions, one for the broad case and one for the medium one. Let \textsc{php-sb} be the $\mathsf{NEXP}$-complete problem: {\em given a sequence ${\bf s}$ $=$ $I_1,...,I_{m+1}$ of instances over the binary relation $\mathsf{br}$ with $\Dom_{I_i} \cap \Dom_{I_{j\neq i}}\!=\! \emptyset$, is there any homomorphism from $I_1\! \otimes\!...\!\otimes\!I_{m}$ to $I_{m+1}$?} Let \textsc{{\footnotesize3}-col} be the $\mathsf{NP}$-complete problem: {\em given a graph $G$ $=$ $(V,E)$, is there a map $\lambda: V \rightarrow \{\mathtt{r},\mathtt{g},\mathtt{b}\}$ such that $\{u,v\}\!\in\!E$ implies $\lambda(u)\!\neq\!\lambda(v)$?}

In the broad case, we are going to show that \textsc{php-sb} $\leq$ \textsc{ess}$_k$, for each $k\!>\!0$.
 To this end, we adapt a known technique already presented in~\cite{DBLP:conf/icdt/CateD15}:
given an input $\mathbf{s}$ for \textsc{php-sb}, consider the map $\mathbf{s} \mapsto (\kb, \unit^k, \langle \mathsf{a}_{m+1}\rangle^k)$, where
$\kb = (K,\varsigma)$,
$K = (\db,O)$,
$\db = \db_1 \cup \dots \cup \db_4$,
$I_{m+2}=\{\mathsf{br}(\mathsf{a}_{m+2},\mathsf{a}_{m+2})\}$,
$\db_1 = I_{1} \cup \dots \cup I_{m+2}$,
$\db_2 = \{\top(c) : c \in \Dom_{\db_1}\}$,
$\db_3 = \{\top(\mathsf{a}_i),\mathsf{br}(\mathsf{a}_{i},c) : c \in \Dom(I_i)\}_{i \in[m+2]}$,
$\db_4 = \mathit{tw}(\{\mathsf{a}_1,\dots,\mathsf{a}_{m+2}\},k)$,
$O=\emptyset$,
$\varsigma$ always selects $D^{\top}$, and the unit $\unit$ is the set
$\{\langle \mathsf{a}_{m+2} \rangle\} \cup \{\langle \mathsf{a}_1 \rangle,\dots,\langle \mathsf{a}_m \rangle\}$.

The reason why we are adding the element $\langle \mathsf{a}_{m+2} \rangle$ to our unit is linked to the fact that from the point of view of possible homomorphisms it does not change anything, however, even if the starting $m$ had been equal to one, it guarantees us that there can be no use of any constant within the \textcolor{black}{conjunctive} formulas that we are going to build. It is, in fact, easy to check that there is a homomorphism from $I_1\! \otimes\!...\!\otimes\!I_{m}$ to $I_{m+1}$ if, and only if, there is a homomorphism from $I_1\! \otimes\!...\!\otimes\!I_{m}\!\otimes\!I_{m+2}$ to $I_{m+1}$. By Lemma~\ref{lem:twins}, we may confine ourselves to the unary case. Let $S$ denote $I_1
		\otimes...\otimes
		I_m\otimes I_{m+2}$. In one direction, if $S \mapsto I_{m+1}$ then, by homomorphism preservation, any conjunctive formula that explains the elements of $\unit$ will also explain $\langle a_{m+1}\rangle$. On the other hand, if $S \not\mapsto I_{m+1}$, then we can construct a formula $\varphi\in\WCQ$ that explains $\unit$ but not $\langle a_{m+1}\rangle$.
To this aim, consider the \textcolor{black}{conjunctive} formulas $$\varphi_S = \bigwedge_{\mathsf{br}(t_1,t_2)\in S}\mathsf{br}(y_{t_1},y_{t_2}) \mbox{~~~~~~~and~~~~~~~} \varphi = x\leftarrow \varphi_S\wedge\bigwedge_{y\in\Dom_{\varphi_S}}\mathsf{br}(x,y).$$
By construction, $\varphi$ explains all the elements of $\unit$ but does not explain $\langle a_{m+1} \rangle$. In order to see this, on the one hand consider, for each $i\in[m+2]\setminus\{m+1\}$, $h_i$ as the projection onto the $i$-th component. Obviously, each of these is a homomorphism from $I_1
		\otimes...\otimes
		I_m\otimes
		I_{m+2} \mapsto I_i$, and in particular, this shows that $\varphi_S$ is true when evaluated on every $I_i$ with $i\in[m+2]\setminus\{m+1\}$, implying $\langle a_{i}\rangle\in\mathit{inst}(\varphi,\kb)\forall i \in [m+2]\setminus\{m+1\}$. On the other hand, suppose by contradiction that there exists a \textcolor{black}{$\C$}-homomorphism $h$ $\mathit{atm}(\varphi)$ to $\varsigma(a_{m+1})$ such that $\langle h(x) \rangle=\langle a_{m+1} \rangle $, then the same \textcolor{black}{$\C$}-homomorphism $h$ limited to variables other than $x$ would show the existence of a homomorphism $h'$ from $I_1
		\otimes...\otimes
		I_m\otimes I_{m+2}$ to $I_{m+1}$, and this is absurd. The summary selector of the constructed SKB is as in Remark~\ref{rem:selector}.

We have just completed the proof concerning the broad case, and we can, therefore, proceed to the next case. We will now show that for each $k\!>\!0$, in the medium case, \textsc{{\footnotesize3}-col} $\leq$ \textsc{ess}$_k$ via the following map: $G$ $\mapsto$ $(\kb, \unit^k, \langle \mathsf{b}_1\rangle^k)$, where:

 \(\kb = (K, \varsigma)\), with \(K = (\db, O)\), \(\varsigma\) always selects \(D\), and \(O = \emptyset\). The database \(\db\) is defined as \(\db = \db_1 \cup \db_2 \cup \db_3 \cup \db_4 \cup \db_5\), where:
\(\db_1 = \{ \mathsf{arc}(u_i, v_i), \mathsf{arc}(v_i, u_i) \mid \{u, v\} \in E,\, i \in [2] \}\),
\(\db_2 = \{ \top(c) \mid c \in \Dom_{\db_1} \}\),
\(\db_3 = \{ \top(\mathsf{b}_i), \mathsf{arc}(\mathsf{b}_i, \mathsf{b}_z) \mid i, z \in [4],\, i \neq z \}\),
\(\db_4 = \{ \top(\mathsf{a}_i), \mathsf{arc}(\mathsf{a}_i, v_i), \mathsf{arc}(v_i, \mathsf{a}_i) \mid v \in V,\, i \in [2] \}\),
\(\db_5 = \mathit{tw}(\Dom_{\db_3} \cup \Dom_{\unit}, k)\),
and \(\unit = \{ \langle \mathsf{a}_1 \rangle, \langle \mathsf{a}_2 \rangle \}\).

By Lemma~\ref{lem:twins}, we may confine ourselves to the unary case.

Let $K_4=(V_4,E_4)$ denote a representation of the complete graph of order four. Note that $\db_3$ translates $K_4$ into a dataset, whereas $\db_1~\cup~\db_4$ translates two disjoint objects isomorphic to $G^a$, call them $G^{a_1}$ and $G^{a_2}$, into a dataset. Let $\tilde{a}\in\{a_1,a_2\}$. On the one hand, say that $G$ is 3-colorable. From Lemma~\ref{lem:color1}, $G^{\tilde{a}}$ is 4-colorable, so there exists a homomorphism $h$ from $G^{\tilde{a}}$ to $K_4$. This homomorphism will map $\tilde{a}$ to a certain element of $V_4$, and since elements from $V_4$ are equal up to isomorphism, w.l.o.g. $h(\tilde{a})=\mathsf{b}_1$. Now let $\varphi$ be a characterization for $\mathit{ess}(\unit,\kb)$. This means that there exists a \textcolor{black}{$\C$}-homomorphism $h_1: \mathit{atm}(\varphi)\mapsto \Dom(\kb)$ such that $h_1(x)=\mathsf{\tilde{a}}$. But then, by \textcolor{black}{$\C$}-homomorphism composition, $h_2: \mathit{atm}(\varphi)\mapsto \Dom(\kb)$ defined as $h\circ h_1$ is another \textcolor{black}{$\C$}-homomorphism such that $h_2(x)=\mathsf{b}_1$, and so $\langle \mathsf{b}_1\rangle\in \mathit{ess}(\unit,\kb)$.

On the other hand, now say that $\langle \mathsf{b}_1\rangle\in \mathit{ess}(\unit,\kb)$. This means that $\mathit{can}(\mathit{ess}(\unit,\kb),\kb)$ explains the unit $\{\langle \mathsf{b}_1\rangle,\ldots,\langle \mathsf{b}_4\rangle\}$, since the four constants $\mathsf{b}_1,\ldots,\mathsf{b}_4$ play symmetric roles in $\db_3$. In particular, this gives rise to a homomorphism from $G^{\tilde{a}}$ to $K_4$. But then, from Lemma~\ref{lem:color1}, $G$ is 3-colorable. The summary selector of the constructed SKB is as in Remark~\ref{rem:selector}.
\end{proof}

\begin{algorithm}[t!]
	\DontPrintSemicolon
	\KwInput{$\kb  = (K, \varsigma)$, $\unit = \{\tau_1,...,\tau_m\}$ and a tuple $\tau$.}
	$\varphi := \mathsf{BuildCan}(\unit,\kb)$\;
	{\bf guess} a mapping $h$ from $\Dom_{\mathit{atm}(\varphi)}$ to $\Dom_{\varsigma(\tau)}$\;

	\If{$h$ \mbox{{\rm is a }}$\C$\mbox{{\rm -homomorphism from}} $\mathit{atm}(\varphi)$ \mbox{{\rm to}} $\varsigma(\tau)$}
	{
	    \If{$\langle h(x_{{\bf s}_1}),...,h(x_{{\bf s}_n}) \rangle = \tau$}
	    {$\mathsf{accept~this~branch}$ }
	 }

	$\mathsf{reject~this~branch}$

	\caption{$\mathsf{InEss}(\unit,\kb,\tau)$}\label{alg:InEss}
\end{algorithm}

\section{\textcolor{black}{ Deciding and Classifying the Navigation Tasks}}\label{sec:tasks}

\textcolor{black}{This section completes the analysis of \textsc{sim}, \textsc{prec} and \textsc{inc}: it first establishes the three membership results, and then proves one hardness result per task.}

\textcolor{black}{One convention is needed for the handy assumption. Each of the two membership tests of Proposition~\ref{taskViaEssEasy} concerns a unit of $|\unit|+1$ tuples, and the bound of the handy assumption depends on the number of tuples of the unit.
Whenever one of the three tasks is considered under the handy assumption, that bound is therefore applied with $m = |\unit|+1$.
The broad and the medium assumption need no such convention: the first one does not depend on that number, and a bound on the number of tuples of $\unit$ is also a bound on the number of tuples of $\unit \cup \{\tau\}$.}

\begin{proposition}\label{prop:membership}
\textcolor{black}{The following hold.
\begin{itemize}
    \item \textsc{sim} is in $\mathsf{NEXP}$ in the broad case, in $\mathsf{NP}$ in the medium case, and in $\mathsf{P}$ in the handy case.
    \item \textsc{inc} is in $\mathsf{coNEXP}$ in the broad case, in $\mathsf{coNP}$ in the medium case, and in $\mathsf{P}$ in the handy case.
    \item \textsc{prec} is in $\mathsf{DEXP}$ in the broad case, in $\mathsf{DP}$ in the medium case, and in $\mathsf{P}$ in the handy case.
\end{itemize}}
\end{proposition}
\begin{proof}
\textcolor{black}{Let $\iota = (\kb, \unit, \tau, \tau')$ be an input for the three tasks, and consider the two languages
$$L_1 = \{\iota : \tau \in \mathit{ess}(\unit \cup \{\tau'\}, \kb)\} \qquad \text{and} \qquad L_2 = \{\iota : \tau' \in \mathit{ess}(\unit \cup \{\tau\}, \kb)\}.$$
The mapping $\iota \mapsto (\kb, \unit \cup \{\tau'\}, \tau)$ is a many-one reduction from $L_1$ to \textsc{ess}, and the mapping $\iota \mapsto (\kb, \unit \cup \{\tau\}, \tau')$ is a many-one reduction from $L_2$ to \textsc{ess}.
Each of the two mappings copies its input and adds one tuple to the unit, hence each of them is computable in linear time.
Each of them also preserves the assumption under which it is applied: the knowledge base and the summary selector are left untouched, so neither the maximum arity of the predicates nor the cost of computing a summary changes; the number of tuples of the unit grows by one; and the bound on the size of the summaries is the one stated for units of $|\unit|+1$ tuples, according to the convention stated at the beginning of this section.
By Theorem~\ref{thm:essfinal}, both $L_1$ and $L_2$ belong to $\mathsf{NEXP}$ in the broad case, to $\mathsf{NP}$ in the medium case, and to $\mathsf{P}$ in the handy case.
By Proposition~\ref{taskViaEssEasy}, the three tasks are the three languages
$$\textsc{sim} = L_1 \cap L_2, \qquad \textsc{inc} = \overline{L_1} \cap \overline{L_2}, \qquad \textsc{prec} = L_1 \cap \overline{L_2},$$
where $\overline{L}$ is the complement of $L$ within the set of the well-formed inputs, which is recognisable in polynomial time.
For \textsc{sim}, both $\mathsf{NEXP}$ and $\mathsf{NP}$ are closed under intersection, since a machine for $L_1 \cap L_2$ guesses the two certificates at once and verifies both of them.
For \textsc{inc}, the language $\overline{L_1} \cap \overline{L_2}$ is the complement of $L_1 \cup L_2$, and both $\mathsf{NEXP}$ and $\mathsf{NP}$ are closed under union; hence $L_1 \cup L_2$ belongs to $\mathsf{NEXP}$ in the broad case and to $\mathsf{NP}$ in the medium case, and its complement belongs to $\mathsf{coNEXP}$ and to $\mathsf{coNP}$, respectively.
For \textsc{prec}, recall from Section~\ref{sec:basics} that $\mathsf{DP}$ is the closure of $\mathsf{NP} \cup \mathsf{coNP}$ under intersection, and that $\mathsf{DEXP}$ is defined in the same way from $\mathsf{NEXP} \cup \mathsf{coNEXP}$. In the medium case $L_1$ belongs to $\mathsf{NP}$ and $\overline{L_2}$ belongs to $\mathsf{coNP}$, so that $L_1 \cap \overline{L_2}$ belongs to $\mathsf{DP}$; in the broad case the same argument gives $\mathsf{DEXP}$.
In the handy case $L_1$ and $L_2$ both belong to $\mathsf{P}$, which is closed under intersection, union and complement, so that all three tasks belong to $\mathsf{P}$.}
\end{proof}

\subsection{\textcolor{black}{Equality of Essential Expansions (\textsc{sim} problem)}}\label{sec:SIM}

The idea behind the \textsc{sim} decision problem is to be able to know whether two tuples, both different and not contained in the input unit, share the same nexus of similarity with the existing unit.
\textcolor{black}{The reduction below duplicates the constant $c$ of the first tuple of the unit, so that the fresh constant $\mathsf{alias}$ shares its facts, and gives $\langle\mathsf{alias}\rangle$ the summary of $\langle c \rangle$. The new tuple then belongs to the essential expansion of the unit by construction, and the question whether it shares with the unit the same nexus of similarity as the tuple under test collapses to the question whether the latter belongs to the essential expansion of the original unit. Note that the summary selector is no longer the one received as input: it has to be defined, and computable in polynomial time, on every tuple of the new domain.}

\begin{theorem}\label{thm:sim}
\textsc{sim} is $\mathsf{NEXP}$-complete, $\mathsf{NP}$-complete and in $\mathsf{P}$ in the broad, medium and handy case, respectively. In particular, for each $k>0$, lower bounds hold already for \textsc{sim}$_k$.
\end{theorem}

\begin{proof}
\textcolor{black}{The upper bounds are those of Proposition~\ref{prop:membership}. We therefore turn to the lower bounds. The construction below is given once and applied twice: to an instance of the broad case it returns an instance of the broad case, and to an instance of the medium case it returns an instance of the medium case.}
 Let $+\textsc{ess}^{\emptyset}_1$ be the decision problem \textsc{ess} with the further hypothesis that the input unit $\unit$ is unary, its cardinality is bigger than one, and the input ontology $O$ is equal to the empty set. Then it holds that $\textsc{ess}^{\emptyset}_1 \leq \textsc{sim}_k$ via the map $(\kb, \unit, \tau) \mapsto (\kb', \unit^k, \tau^k, \langle \mathsf{alias}\rangle^k)$, where:
          \(\kb = (K, \varsigma)\), with \(K = (\db, O)\), \(\unit = \{\tau_1, \dots, \tau_m\}\) for some \(m > 1\), and \(\tau_1\) of the form \(\langle c \rangle\). Let \(\kb' = (K', \varsigma')\), with \(K' = (D', O')\), where \(D' = \bar{D} \cup \mathit{tw}(\Dom_{\bar{D}}, k)\), \(\bar{D} = \mathit{double}(D, c)\), and \(O = O' = \emptyset\). The selector \(\varsigma'\) satisfies \(\varsigma'(K', \tau') = \varsigma(K, \tau_0) \cup \mathit{tw}(\Dom_{\{\tau_0\}}, k)\), where \(\tau' \in \Dom_{D'}^k\) and \(\tau_0 = \mathit{off}(\tau', c)\).
By Lemma~\ref{lem:twins}, we may confine ourselves to the unary case.

The first thing the reader has to notice is that $\langle\mathsf{alias}\rangle\in\mathit{ess}(\unit,\kb')$ by construction. In fact, say that $h$ is a $\C$-homomorphism from $\mathit{atm}(\mathit{can}(\mathit{ess}(\unit,\kb),\kb))$ to $\varsigma'(K',\langle c \rangle)$ such that $x_1 \mapsto c$. Please note that this $\C$-homomorphism must exist by hypotheses. Then we can construct another $\C$-homomorphism from $\mathit{atm}(\mathit{can}(\mathit{ess}(\unit,\kb),\kb))$ to $\varsigma'(\db',\langle \mathsf{alias} \rangle)$, call it $\tilde{h}$, such that $x_1 \mapsto \mathsf{alias}$ as follows: $\tilde{h}(y) = h(y)$ if $y \neq x_1$ and $\tilde{h}(x_1) = \mathsf{alias}$. That being said, the problem collapses to whether or not $\tau\in\mathit{ess}(\unit,\kb)$, and if that is the case, then we will end up with a yes answer since the same $\C$-homomorphism that was a witness for $\tau$ to be in $\mathit{ess}(\unit,\kb)$ is still a witness for $\tau$ to be in $\mathit{ess}(\unit,\kb')$ and vice versa. Having a $\C$-homomorphism that is a witness for $\tau$ to be in $\mathit{ess}(\unit,\kb')$ leads us to a witness for $\tau$ to be in $\mathit{ess}(\unit,\kb)$.

The summary selector of the constructed SKB is as in Remark~\ref{rem:selector}.
\end{proof}

\subsection{\textcolor{black}{Strict Inclusion of Essential Expansions (\textsc{prec} problem)}}\label{sec:PREC}
The idea behind the \textsc{prec} decision problem is to be able to know whether two tuples, both different and not contained in the input unit, share different nexus of similarity with the given unit, with the further hypothesis that in particular the first tuple shares higher nexus of similarity with the unit.
\textcolor{black}{The proof rests on a property of the explanations that arise from a knowledge base whose dataset splits into two parts sharing no predicate other than $\top$: every such explanation is separable, in the sense defined below, so that the unit and the two tuples under test can be handled one part at a time. This is also why the construction needs units of arity at least two: the unit is the Cartesian product of one unit per part, and each part needs at least one position.}

\begin{theorem}\label{thm:prec}
\textsc{prec} is $\mathsf{DEXP}$-complete, $\mathsf{DP}$-complete and in $\mathsf{P}$ in the broad, medium and handy case, respectively. In particular, for each $k>1$, lower bounds hold already for \textsc{prec}$_k$.
\end{theorem}

\begin{proof}
Before proceeding with the proof, it is helpful to introduce the following definition and establish several key properties that will play a crucial role in the argument. Let ${\bf x}_1$ and ${\bf x}_2$ be the sequences $x_1,\dots,x_{n_1}$ and $x_1',\dots,x_{n_2}'$ respectively. A formula $\varphi\in\WCQ$ is said \emph{separable} if the following two conditions apply: the formula is of the form $ {\bf x}_1, {\bf x}_2 \leftarrow \mathit{conj}(\varphi_1),\mathit{conj}(\varphi_2)$ where $\mathit{dom}(\mathit{atm}(\varphi_1))\cap \mathit{dom}(\mathit{atm}(\varphi_2))=\emptyset$ and both $ {\bf x}_1 \leftarrow \mathit{conj}(\varphi_1)$ and $ {\bf x}_2 \leftarrow \mathit{conj}(\varphi_2)$ are in $\WCQ$.
Consider a separable \textcolor{black}{conjunctive} formula $\varphi$ of the form $ {\bf x}_1, {\bf x}_2 \leftarrow \mathit{conj}(\varphi_1),\mathit{conj}(\varphi_2)$, a $|{\bf x}_1|$-ary tuple $\langle {\bf c}_1 \rangle$, a $|{\bf x}_2|$-ary tuple $\langle {\bf c}_2 \rangle$, and a selective knowledge base $\kb$. If $\varsigma(\langle {\bf c}_1, {\bf c}_2 \rangle)=\varsigma(\langle {\bf c}_1\rangle)~\cup~\varsigma(\langle {\bf c}_2\rangle)$, and $\Dom_{\varsigma(\langle {\bf c}_1 \rangle)}\cap\Dom_{\varsigma(\langle {\bf c}_2 \rangle)}=\emptyset$, then it holds that $\langle {\bf c}_1, {\bf c}_2 \rangle \in \mathit{inst}(\varphi,\kb)$ if, and only if, $\langle {\bf c}_1 \rangle\in\mathit{inst}(\varphi_1,\kb)$ and $\langle {\bf c}_2\rangle\in\mathit{inst}(\varphi_2,\kb)$.
This observation follows directly from the construction and the assumption that $\varsigma(\langle {\bf c}_1, {\bf c}_2 \rangle)=\varsigma(\langle {\bf c}_1\rangle)~\cup~\varsigma(\langle {\bf c}_2\rangle)$, which enables us to reason independently on the two components. Specifically, we can define two distinct $\C$-homomorphisms—denoted $h_1$ and $h_2$—the first from $\mathit{atm}(\varphi_1)$ to $\varsigma(\langle {\bf c}_1 \rangle)$, and the second from $\mathit{atm}(\varphi_2)$ to $\varsigma(\langle {\bf c}_2 \rangle)$. The union of these mappings naturally induces a third $\C$-homomorphism from $\mathit{atm}(\varphi)$ to $\varsigma(\langle {\bf c}_1, {\bf c}_2 \rangle)$. For the converse direction, consider a $\C$-homomorphism $h$ from $\mathit{atm}(\varphi)$ to $\varsigma(\langle {\bf c}_1, {\bf c}_2 \rangle)$ and an atom $\alpha$ of $\varphi_1$: since $ {\bf x}_1 \leftarrow \mathit{conj}(\varphi_1)$ belongs to $\WCQ$, the atom $\alpha$ is connected to some variable of ${\bf x}_1$, which $h$ maps to a constant of ${\bf c}_1$; two adjacent atoms share a term and the two summaries have disjoint domains, so the whole connected structure is mapped inside $\varsigma(\langle {\bf c}_1 \rangle)$, and the restriction of $h$ to $\mathit{atm}(\varphi_1)$ witnesses $\langle {\bf c}_1 \rangle \in \mathit{inst}(\varphi_1,\kb)$. We refer to this statement as $(i)$.

Consider now a separable \textcolor{black}{CF} $\varphi$ of the following form $ {\bf x}_1, {\bf x}_2 \leftarrow \mathit{conj}(\varphi_1),\mathit{conj}(\varphi_2)$, a $(|{\bf x}_1|+|{\bf x}_2|)$-ary $\unit$, and a selective knowledge base $\kb$. Let $\unit_1$ be the $|{\bf x}_1|$-ary unit of the form $\{ \tau ~:~ \exists \tau' \ s.t \ \tau\circ\tau' \in \unit \}$ and $\unit_2$ be the $|{\bf x}_2|$-ary unit of the form $\{ \tau ~:~ \exists \tau' \ s.t \ \tau'\circ\tau \in \unit \}$. If $\forall \tau\in\unit$ we have $\varsigma(\tau)=\varsigma(\tau')~\cup~\varsigma(\tau'')$ and $\varsigma(\tau')\cap\varsigma(\tau'')=\emptyset$, where $\tau=\tau'\circ \tau''$ with $\tau'\in\unit_1$ and $\tau''\in\unit_2$, then it holds that $\varphi$ is a characterization for $\unit$ if, and only if, $\varphi_1$ is a characterization for $\unit_1$ and $\varphi_2$ is a characterization for $\unit_2$, provided that $\unit = \unit_1 \times \unit_2$, which is the case in the construction below. To see this, note first that every explanation of $\unit$ is separable: each of its atoms is connected to a free variable, the summary of a tuple of $\unit$ is the union of the summaries of its two halves, and those two summaries have disjoint domains, so no connected structure of an explanation meets both of them. By $(i)$, a separable formula explains $\unit$ if, and only if, its two parts explain $\unit_1$ and $\unit_2$. Hence, if $\varphi_1$ and $\varphi_2$ are characterizations, then every explanation of $\unit$ maps into $\varphi$ one part at a time; and conversely, if one part of $\varphi$ were not a characterization of the corresponding unit, replacing it by a more specific explanation of that unit would produce an explanation of $\unit$ that does not map into $\varphi$, against the assumption that $\varphi$ is a characterization. We refer to this statement as $(ii)$.

\textcolor{black}{The upper bounds are those of Proposition~\ref{prop:membership}. We therefore turn to the lower bounds. The construction below is again given once and applied to both cases, and it is moreover uniform over all the languages of $\mathsf{DEXP}$ and of $\mathsf{DP}$: it takes an arbitrary language of one of these two classes and produces, from each of its words, an instance of \textsc{prec}.}

Let +\textsc{ess}$^{\emptyset}_1$ be the decision problem \textsc{ess} with the further hypothesis that the input unit $\unit$ is unary, its cardinality is bigger than one, and the input ontology $O$ is equal to the empty set. Let +\textsc{ess}$^{\emptyset}_{k-1}$ be the decision problem \textsc{ess} with the further hypothesis that the input unit $\unit$ is $(k-1)$-ary, its cardinality is bigger than one, and the input ontology $O$ is equal to the empty set.

For each $L\in \mathsf{DEXP}$ (resp., $\mathsf{DP}$), $L \leq$ \textsc{prec}$_k$ in the broad (resp., medium) case, via the mapping $w \mapsto (\kb, \unit, \tau', \tau'')$, where $w$ is any string over the alphabet of $L$, $L = L_1 \cap L_2$, $L_1 \in \mathsf{NEXP}$ (resp., $\mathsf{NP}$), $L_2 \in \mathsf{coNEXP}$ (resp., $\mathsf{coNP}$), $\rho_1$ is a reduction from $L_1$ to +\textsc{ess}$^{\emptyset}_1$ in the broad (resp., medium) case, $\rho_2$ is a reduction from $L_2$ to $\neg$+\textsc{ess}$^{\emptyset}_{k-1}$ in the broad (resp., medium) case, each $\rho_i(w) = (\kb_i, \unit_i, \tau_i)$, each $\kb_i = (K_i,\varsigma_i)$, each $K_i = (\db_i,\emptyset)$, without loss of generality each constant and predicate different from $\top$ in $\db_1$ (resp., $\db_2$) has $a$- (resp., $b$-) as a prefix, each $\unit_i = \{\tau_{i,1},...,\tau_{i,m_i}\}$, $\kb = (K,\varsigma)$, $K = (\db,\emptyset)$, $\db = \db_1 \cup \db_2$, $\varsigma$ is such that $\varsigma(K,\tau)=\varsigma_1(K_1,\tau|_{\Dom(\db_1)})~\cup~ \varsigma_2(K_2,\tau|_{\Dom(\db_2)})$, $\unit = \unit_1 \times \unit_2$, $\tau' = \tau_1 \circ \tau_{2,1}$, $\tau'' = \tau_{1,1} \circ \tau_2$, each $\tau|_{\Dom(\db_i)}$ is obtained from $\tau$ by removing the terms not in $\Dom(\db_i)$.

The core idea behind the construction is that the final knowledge base we build is bipartite in two distinct respects. First, the bipartition is evident at the domain level, but more crucially, the predicate $\top$ is the only symbol that constants coming from the two separate reduction actually shares. This means that constants generated by the first reduction are entirely distinct—not only from those produced by the second reduction, but also in terms of the predicates they occur with.
This structural separation implies that any explanation derived from the construction will be necessarily separable, in the sense previously defined.
From this observation, the correctness of the reduction follows from $(i)$ and $(ii)$. The summary selector of the constructed SKB is as in Remark~\ref{rem:selector}.
\end{proof}

\subsection{\textcolor{black}{Incomparability of Essential Expansions (\textsc{inc} problem)}}\label{sec:INCS}

The idea behind the \textsc{inc} decision problem is to be able to know whether two tuples, both different and not contained in the input unit, share incomparable nexus of similarity with the existing unit.
\textcolor{black}{The construction below is the one used for \textsc{sim} with one difference: every constant of the original dataset receives the predicate $\mathsf{focus}$, which therefore occurs in every characterization of the unit, whereas the fresh constant $\mathsf{alias}$, which inherits the summary of the constant it duplicates, does not receive it. The formula $x \leftarrow \mathsf{focus}(x)$ then explains the unit but not $\langle\mathsf{alias}\rangle$, so that $\langle\mathsf{alias}\rangle$ stays outside the essential expansion of the unit and the two tuples end up in the position described in Section~\ref{sec:overview}: neither of them shares with the unit everything the other shares.}

\begin{theorem}\label{thm:inc}
\textsc{inc} is $\mathsf{coNEXP}$-complete, $\mathsf{coNP}$-complete and in $\mathsf{P}$ in the broad, medium and handy case, respectively. In particular, for each $k>0$, lower bounds hold already for \textsc{inc}$_k$.
\end{theorem}

\begin{proof}
\textcolor{black}{The upper bounds are those of Proposition~\ref{prop:membership}. We therefore turn to the lower bounds. Here the construction, again given once for both cases, reduces \textsc{ess} to the complement of \textsc{inc}: an instance of \textsc{ess} is accepted if, and only if, the instance of \textsc{inc} that the construction returns is rejected.}
 Let +\textsc{ess}$^{\emptyset}_1$ be the decision problem \textsc{ess} with the further hypothesis that the input unit $\unit$ is unary, its cardinality is bigger than one, and the input ontology $O$ is equal to the empty set, then it holds that +\textsc{ess}$^{\emptyset}_1$ $\leq$ $\neg$\textsc{inc}$_k$ via the following map: $(\kb, \unit, \tau) \mapsto (\kb', \unit^k, \tau^k, \langle \mathsf{alias}\rangle^k)$, where :
 \(\kb = (K, \varsigma)\) with \(K = (D, O)\), \(O = \emptyset\), and \(\unit = \{\tau_1, \dots, \tau_m\}\) for some \(m > 1\), where \(\tau_1\) is of the form \(\langle c \rangle\). Let \(\kb' = (K', \varsigma')\) with \(K' = (D', O')\), \(O' = \emptyset\), and \(D' = \bar{D} \cup \mathit{tw}(\Dom_{\bar{D}}, k) \cup \mathit{fc}(\Dom_D)\), where \(\bar{D} = \mathit{double}(D, c)\). The selector \(\varsigma'\) satisfies \(\varsigma'(K', \tau') = \varsigma(K, \tau_0) \cup \mathit{tw}(\Dom_{\{\tau_0\}}, k) \cup \mathit{fc}(\Dom_D)\), where \(\tau' \in \Dom_{D'}^k\) and \(\tau_0 = \mathit{off}(\tau', c)\).

By Lemma~\ref{lem:twins}, we may confine ourselves to the unary case.

By construction, we know that $\langle \mathsf{alias} \rangle$ for sure is not in $\mathit{ess}(\unit,\kb')$ and this is because the \textcolor{black}{CF} $x\leftarrow\mathsf{focus}(x)$ explains every element in $\unit$ but does not explain $\langle \mathsf{alias} \rangle$. This lets us know that a characterization for $\unit$ will use that predicate and so it will not explain $\langle\mathsf{alias}\rangle$. That being said, the problem collapses to whether or not $\tau\in\mathit{ess}(\unit,\kb)$, and if that is the case then we will end up with a no answer since the same $\C$-homomorphism that was a witness for $\tau$ to be in $\mathit{ess}(\unit,\kb)$ can be enlarged to a witness for $\tau$ to be in $\mathit{ess}(\unit,\kb')$ and vice versa. Having a $\C$-homomorphism that is a witness for $\tau$ to be in $\mathit{ess}(\unit,\kb')$ leads us to a witness for $\tau$ to be in $\mathit{ess}(\unit,\kb)$.

The summary selector of the constructed SKB is as in Remark~\ref{rem:selector}.
\end{proof}

\section{Conclusions and Future Work}
\label{sec:conclusions}

In this paper, we proposed and studied ---under different computational assumptions--- key reasoning tasks for navigating expansion graphs according to the formal semantics of the logical framework by Amendola et al.~\cite{AMENDOLA2024120331}.
From our results, we can establish a clear dichotomy between the intractability of the general \emph{broad} case and the full tractability of the more realistic \emph{handy} one, where all tasks become polynomial-time computable.
\textcolor{black}{Some questions remain open on the theoretical side. In the broad case the bounds for \textsc{can} and for \textsc{core} do not match, and in the medium case \textsc{core} is shown to lie outside $F\mathsf{P}$ only under the assumption that $\mathsf{NP}
\neq\mathsf{coNP}$, while its upper bound is $F\mathsf{P}^{\mathsf{NP}}$. The lower bounds for \textsc{prec} are established for units of arity at least two, and the unary case is open. Finally, when the bound on the arity of predicates is dropped, our lower bounds still hold, whereas the upper bounds would have to be revisited.}

The tractability of the \emph{handy} case has significant practical implications and provides a clear road-map for future development.
We envision a real-time system where users can dynamically navigate an expansion graph to compare entities according to the nexus of similarity they share with the group of entities given as input.
This effort will also involve systematic experimentation to test and compare different summary selectors and the definition of appropriate metrics to quantitatively evaluate the ``nexus of similarity''. For scenarios where tractability is not guaranteed, such as the \emph{broad} setting, a crucial research avenue is the design of approximation algorithms that can provide heuristically sound answers in reasonable time, trading completeness for scalability.

Beyond these practical applications, our work opens up several promising theoretical directions. A primary goal is to identify new structural properties, beyond the \emph{handy} case, that reflect real-world knowledge domains while preserving computational tractability. Furthermore, we plan to investigate the expressiveness-tractability trade-off by considering alternative explanation languages. While the current query language is effective, exploring the impact of adopting fragments of Description Logics, acyclic conjunctive queries, or queries with comparison operators is essential. Another vital line of research involves enriching summaries with intentional knowledge, in this way the framework could capture more nuanced nexus of similarity, even for entities with incomplete data. These extensions, combined with a broader computational analysis of additional reasoning tasks, will significantly advance the state of the art in structured knowledge exploration.

\section*{Acknowledgments}

\noindent This work was supported by the European Union -- NextGenerationEU under the National Recovery and Resilience Plan (NRRP) through the following projects:
\textbf{FAIR} (PE\_00000013, CUP H23C22000860006, Spoke 9 -- WP9.1 ``KRR Frameworks for Green-aware AI'');
\textbf{Tech4You} (ECS00000009, CUP H23C22000370006, Spoke 6); and
\textbf{RAISE} (ECS00000035, CUP H53C24000400006, sub-project GOLD).

\noindent We also acknowledge support from the Italian Ministry of University and Research (MUR) via the PRIN project \textbf{PRODE} (ID 20224C9HXA, CUP H53D23003420006).

\noindent Finally, funding was provided by the Ministry of Enterprises and Made in Italy (MIMIT) under the projects:
\textbf{ASVIN} (F/360050/02/X75, CUP B29J24000200005);
\textbf{STROKE 5.0} (F/310031/02/X56, CUP B29J23000430005); and
\textbf{EI-TWIN} (F/310168/05/X56, CUP B29J24000680005).
\bibliographystyle{alphaurl}
\bibliography{biblio_extended.bib}

\end{document}